%% file: corrintro.tex
\documentclass[a4paper,11pt]{book}
\usepackage[bindingoffset=0cm,a4paper,centering,textheight=200mm,width=125mm,totalheight=297mm,footskip=0.7cm]{geometry}
\usepackage[longnamesfirst,sectionbib]{natbib}
\usepackage{colortbl,hhline,numberpage,rotating}
\usepackage{fancyhdr}
\newcommand{\ourfooter}[2]{\footnotesize \sf Chapter\ \thechapter\ of
  the {\em Handbook of Epistemic Logic}, H.\ van
  Ditmarsch, J.Y.\ Hal\-pern, W.\ van der Hoek and B.\ Kooi (eds),
  College Publications, 2015, pp.~\pageref{#1}--\pageref{#2}.}
 \newenvironment{proof}{\noindent {\bf Proof} \ \ \ }{\hfill$\dashv$
   \\ }
 \newcommand{\DEL}{\mbox{\bf DEL}}
 \newcommand{\PAL}{\mbox{\bf PAL}}

\usepackage{chapterbib}
\usepackage{amsmath,amssymb,bbm,latexsym,mathrsfs}
\usepackage{makeidx,minitoc,multirow}
\usepackage{graphicx}    
\graphicspath{{CHAPTERS/10EpistemicFoundationsforGames/}}
\usepackage{algorithmic,algorithm}
\usepackage[hypertexnames=false]{hyperref}

\newcommand{\refstart}{
\newpage
\thispagestyle{fancy}
\fancyhead[RO]{\thepage}
\fancyhead[LE]{\thepage}
\fancyhead[RE]{\emph{\leftmark}}
\fancyfoot{}
\fancyfoot[C]{}
\bibliographystyle{chicago} 
\renewcommand{\bibname}{References}
\addcontentsline{toc}{section}{\numberline{}References}
}

\newcommand{\putaway}[1]{}
\newcommand{\size}[1]{\mid\!#1\!\mid}
\newcommand{\msize}[1]{\|#1\|}
\newcommand{\prok}{\ensuremath{\mathbf K}}

\newcommand{\imp}{\rightarrow}

\newcommand{\et}{\wedge}
\newcommand{\vel}{\vee}
\newcommand{\Et}{\bigwedge}

\newcommand{\atom}{p}
\newcommand{\Group}{A}
\newcommand{\group}{\Group}
\newcommand{\agent}{a}
\newcommand{\agenta}{a}

\newcommand{\States}{S}
\newcommand{\state}{s}
\newcommand{\statet}{t}
\newcommand{\Domain}{{\mathcal D}}
\renewcommand{\phi}{\varphi}
\newcommand{\bisim}{{\raisebox{.3ex}[0mm][0mm]{\ensuremath{\medspace \underline{\! \leftrightarrow\!}\medspace}}}}
\newcommand{\bisrel}{\ensuremath{\mathfrak{R}}}
\newcommand{\Nat}{\mathbb N}

\usepackage[amsmath,thmmarks]{ntheorem}
\theorembodyfont{\rmfamily}
\theoremstyle{break}
\theoremsymbol{\ensuremath{\dashv}}

\newtheorem{definition}{Definition}[chapter]
\newtheorem{theorem}{Theorem}[chapter]
\newtheorem{example}{Example}[chapter]
\newtheorem{proposition}{Proposition}[chapter]

\newcommand{\addauthorstotoc}[1]{\addtocontents{toc}{\protect\vspace*{5pt}
    \noindent \protect\emph{#1}\protect\vspace*{5pt}}}
\newcommand{\chapterauthors}[1]{{\begin{center}\large\textbf{#1}
    \end{center}
  } }

\newcommand{\commentout}[1]{}

\newenvironment{abstract}{\begin{quote} \noindent \textbf{Abstract} }{\end{quote}}

\newcommand{\noopsort}[2]{#2}
\newcommand{\lan}[1]{\ensuremath\mathsf{#1}}
\newcommand{\mc}[1]{\ensuremath{\mathcal{#1}}}
\newcommand{\mcx}{\mc{X}}
\newcommand{\mb}[1]{\ensuremath{\mathbf{#1}}}
\newcommand{\mbx}{\mb{X}}
\newcommand{\atoms}{\lan{At}}
\newcommand{\agents}{\lan{Ag}}
\newcommand{\operators}{\lan{Op}}
\renewcommand{\eqref}[1]{(\ref{#1})}
\newcommand{\axiom}[1]{\ensuremath{\mb{#1}}}
\newcommand{\node}[2]{\langle #1 \circ #2 \rangle}

\input{seesindex}
\usepackage{pdfpages}

\author{Hans van Ditmarsch \and Joseph Y.\ Halpern \and Wiebe van der Hoek \and Barteld Kooi}
\title{Handbook of Epistemic Logic}
\date{\today}

\begin{document}
\dominitoc
\faketableofcontents
\chapter{An Introduction to Logics of Knowledge and Belief}
\chaptermark{Introduction}
\label{chap:introduction}
\label{chap1}

\chapterauthors{Hans van Ditmarsch\\ Joseph Y. Halpern\\ Wiebe van der Hoek\\
  Barteld Kooi}   
\addauthorstotoc{Hans van Ditmarsch, Joseph Y. Halpern, Wiebe van der Hoek and
  Barteld Kooi}   
\setcounter{minitocdepth}{1}
\minitoc

\begin{abstract}
This chapter provides an introduction to some basic concepts of epistemic logic, basic formal languages, their semantics, and proof systems. It also contains an overview of the handbook, and a brief history of epistemic logic and pointers to the literature.
\end{abstract}

\thispagestyle{fancy}
\fancyhead{}
\renewcommand{\headrulewidth}{0pt}
\fancyfoot[LE,LO]{\ourfooter{chap1:first page}{chap1:lastpage}}
\fancyfoot[C]{}
\label{chap1:first page}

\newcommand{\chapref}[1]{\arabic{#1}}
%
%
%
%
%
\newcounter{chap:onlyknowing}
\setcounter{chap:onlyknowing}{2}

\newcounter{chap:awareness}
\setcounter{chap:awareness}{3}

\newcounter{chap:knowledgeandtime}
\setcounter{chap:knowledgeandtime}{5}

\newcounter{chap:dynamicepistemiclogic}
\setcounter{chap:dynamicepistemiclogic}{6}

\newcounter{chap:beliefrevisioninDEL}
\setcounter{chap:beliefrevisioninDEL}{7}

\newcounter{chap:probabilisticupdates}
\setcounter{chap:probabilisticupdates}{4}

\newcounter{chap:modelchecking}
\setcounter{chap:modelchecking}{8}

\newcounter{chap:epistemicfoundationsforgames}
\setcounter{chap:epistemicfoundationsforgames}{9}

\newcounter{chap:agentsbdi}
\setcounter{chap:agentsbdi}{10}

\newcounter{chap:strategicability}
\setcounter{chap:strategicability}{11}

\newcounter{chap:knowledgeandsecurity}
\setcounter{chap:knowledgeandsecurity}{12}

\section{Introduction to the Book}\label{chap1:sec:intro}
This introductory chapter has four goals:
\begin{enumerate}
\item\label{chap1:enumitem:one:a} an informal introduction to some
basic concepts of epistemic logic;
\item\label{chap1:enumitem:two:a} basic formal languages, their
semantics, and  proof systems;
\item  an overview of the handbook; and
\item a brief history of epistemic logic and pointers to the
literature.
\end{enumerate}
     
In Section~\ref{chap1:sec:tools}, we deal with the first two items.
We provide examples that should help to
connect the informal concepts with the formal definitions. Although the
informal meaning of the concepts that we discuss may vary from author to
author in this book (and, indeed, from reader to reader), the formal
definitions and notation
provide a framework for the discussion in the remainder of the book.

In Section~\ref{chap1:sec:overview}, we outline how the basic concepts
from this chapter are further developed in subsequent chapters, and
how those chapters relate to each other.
This chapter, like all others, concludes with a section of notes, which
gives all the relevant references and some historical background, and
a bibliography. 

\section{Basic Concepts and Tools}\label{chap1:sec:tools}
As the title suggests, this book uses a formal tool, {\em logic}, to study
the notion of {\em knowledge} (``episteme'' in Greek, hence \emph{epistemic logic}) and belief, and, in a wider sense, the notion of {\em  information}.

Logic is the study of reasoning, formalising the way in which certain
conclusions can be reached, given certain premises. This can be done by
showing that the conclusion can be {\em derived} using some deductive
system (like the axiom systems we present in
Section~\ref{chap1:subsec:axioms}), or by arguing that the {\em truth}
of the conclusion must follow from the truth of the premises (truth is
the concern of the semantical approach of
Section~\ref{chap1:subsec:semantics}). However, first of all, the
premises and conclusions need to be presented in some formal {\em
language}, which is the topic of
Section~\ref{chap1:subsec:language}. Such a language allows us to
specify and verify properties of complex systems of interest.

Reasoning about knowledge and belief, which is the focus of this book,
has subtleties beyond those that arise in propositional or predicate
logic.
%
Take, for instance, the law of excluded
middle in classical logic, which says that for any proposition $p$,
either $p$ or $\neg p$ (the negation of $p$) must hold; formally,
$p\lor \neg p$ is valid. In the language of epistemic logic, we write
$K_\agent p$ for `agent $a$ knows that $p$ is the case'.  Even
this simple addition to the language allows
us to ask many more questions.  For example, which of the following
formulas should
be valid, and how are they related?
What kind of `situations' do the formulas describe?
\begin{itemize}
\item $K_\agent p \lor \neg K_\agent p$
\item $K_\agent p \lor K_\agent\neg p$
\item $K_\agent(p \lor \neg p)$
\item $K_\agent p \lor \neg K_\agent \neg p$
\end{itemize}
It turns out that, given the semantics of interest to us,  only the
first and third formulas above are valid.  Moreover as we will see below, $K_\agent p$ logically implies $\neg K_\agent \neg p$, so the last formula is equivalent
to $\neg K_\agent \neg p$, and says `agent $a$ considers $p$ possible'.
This is incomparable to the second formula, which says agent $a$ knows
whether $p$ is true'.


One of the appealing features of epistemic logic is that
it goes beyond the `factual knowledge'
that 
%
the agents have. Knowledge can be
about knowledge, so we can write expressions like $K_a(K_ap \rightarrow
K_aq)$ ($a$ knows that if he knows that $p$, he also knows that
$q$). More interestingly, we can model knowledge about other's
knowledge, which is important when we reason about communication
protocols. Suppose $Ann$ knows some fact $m$ (`we meet for dinner the
first Sunday of August'). So we have $K_am$. Now suppose Ann e-mails
this message to Bob at Monday 31st of July, and Bob reads it that
evening. We then have $K_bm \land K_bK_am$.  Do we have $K_aK_bm$?
Unless Ann has information that Bob has actually read the message, she
cannot assume that he did, so we have $(K_am \land \neg K_aK_bm \land
\neg K_a\neg K_bm)$.

We also have $K_a K_b \neg K_aK_bm$. To see this, we already noted that
$\neg K_aK_b$ $m$, since Bob might not have read the message yet. But if
{\em we} can deduce that, then Bob can as well (we implicitly assume
that all agents can do perfect reasoning), and, moreover, Ann can deduce
{\em that}. Being a gentleman, Bob should resolve the situation in which
$\neg K_aK_bm$ holds, which he could try to do by replying to Ann's
message. Suppose that Bob indeed replies on Tuesday morning, and Ann reads
this on Tuesday evening. Then, on that evening, we indeed have
$K_aK_bK_am$. But of course, Bob cannot assume Ann read the
acknowledgement, so we have $\neg K_bK_aK_bK_am$. It is obvious that if
Ann and Bob do not want any ignorance about knowledge of $m$, they
better pick up the phone and verify $m$. Using the phone is a good
protocol that guarantees $K_am \land K_bm \land K_aK_bm \land K_bK_am
\land K_aK_bK_am \land \dots$, a notion that we call {\em common
knowledge}; see Section~\ref{subsec:groupnotions}.

The point here is that our formal language helps clarify the effect of
a (communication) protocol on the information of the participating
agents. This is the focus of
Chapter~\chapref{chap:knowledgeandsecurity}. It is
important to note that requirements of protocols can involve both
knowledge and ignorance: in the above example for instance, where
Charlie is a roommate of Bob, a goal (of Bob) for the protocol might
be that he knows that Charlie does \emph{not} know the message ($K_b
\neg K_cm$), while a goal of Charlie might even be $K_cK_b \neg
m$.  Actually, in the latter case, it may be more reasonable to
write $K_cB_b\neg m$: Charlie knows that Bob
believes that there is no dinner on Sunday. A temporal
progression from $K_bm \land \neg K_aK_bm$ to $K_bK_am$ can be viewed as
learning. This raises interesting questions in the study of epistemic
protocols: given an initial and
final specification of information, can we find a sequence of messages
that take us from the former to the latter? Are there optimal such
sequences? These questions are addressed in
Chapter~\chapref{chap:knowledgeandtime}, specifically
Sections~5.7 and 5.9.

Here is an example of a scenario where the question is to derive a
sequence of messages from an initial and final specification of
information. It is taken from
Chapter~\chapref{chap:knowledgeandsecurity}, and it demonstrates that
security protocols that aim to ensure that certain agents stay
ignorant cannot (and do not) always rely on the fact that some messages
are kept secret or hidden.

\begin{quote}
Alice and Betty
each draw three cards from a pack of seven cards, and Eve (the
eavesdropper) gets the remaining card.
Can players Alice and Betty learn each other's cards without
revealing that information to Eve? The restriction is that Alice
and Betty can make only public announcements that Eve can hear.
\end{quote}

We assume that (it is common knowledge that) initially, all three agents
know the composition of the pack of
cards, and each agent knows which cards she holds.
At the end of the protocol, we want Alice and Betty to know which cards each of
them holds, while Eve should know only which cards she (Eve) holds.
Moreover, messages can only be public
announcements (these are formally described in
Chapter~\chapref{chap:dynamicepistemiclogic}), which in this setting just
means that Alice and Betty can talk to each other, but it is common
knowledge that Eve hears them. Perhaps surprisingly, such a protocol
exists, and, hopefully less surprisingly by now, epistemic logic
allows us to formulate precise epistemic conditions, and the kind of
announcements that should be allowed. For instance, no agent is allowed
to lie, and agents can announce only what they know.  Dropping the second
condition would allow Alice to immediately announce Eve's card, for
instance.
Note there is an important distinction here: although Alice knows that
there is an announcement that she can make that would bring
about the desired state of knowledge (namely, announcing Eve's card),
there is not something that Alice knows that she can announce that would
bring about the desired state of knowledge (since does not in fact know
Eve's card).  This distinction has be called the \emph{de
dicto}/\emph{de re} distinction in the literature.  The connections
between knowledge and strategic ability are the topic of
Chapter~\chapref{chap:strategicability}.

Epistemic reasoning is also important in distributed computing. As
argued in Chapter~\chapref{chap:knowledgeandtime}, processes or programs
in a distributed environment often have only a limited view of the global
system initially; they gradually come to know more about the system. Ensuring
that each process has the appropriate knowledge needed in order
to act is the main issue here.
The chapter mentions a number of problems in distributed systems where
epistemic tools are helpful, like agreement problems (the dinner
example of Ann and Bob above would be a simple example) and the
problem of mutual exclusion, where processes sharing a resource must
ensure that only one process uses the resource at a time. An instance
of the latter is provided in Chapter~\chapref{chap:modelchecking}, where
epistemic logic is used to specify a correctness property of the {\em
  Railroad Crossing System}. Here, the agents Train, Gate and
Controller must ensure, based on the type of signals that they send, that
the train is never at the crossing while the gate is
`up'. Chapter~\chapref{chap:modelchecking} is on model checking;
it provides techniques to automatically verify that such properties
(specified in an epistemic temporal language; cf.
Chapter~\chapref{chap:knowledgeandtime}) hold. Epistemic tools to
deal with the problem of mutual exclusion are also discussed in Chapter
\chapref{chap:strategicability},
in the context of dealing with {\em shared file updates}.

Reasoning about knowing what others know (about your knowledge) is
also typical in strategic situations, where one needs to make a
decision based on how others will act (where the others, in turn, are
basing their
decision on their reasoning about you). This kind of scenario
is the focus of
game theory.
\emph{Epistemic} game theory studies game theory 
using notions from
epistemic logic.  (Epistemic game theory is the subject of 
Chapter~\chapref{chap:epistemicfoundationsforgames} in this book.)
Here, we 
give a simplified example of one of the main ideas. Consider the game
in Figure~\ref{chap1:fig:extgame}.
\begin{figure}[h!]\center
\begin{center}
\includegraphics[width=0.45\textwidth]{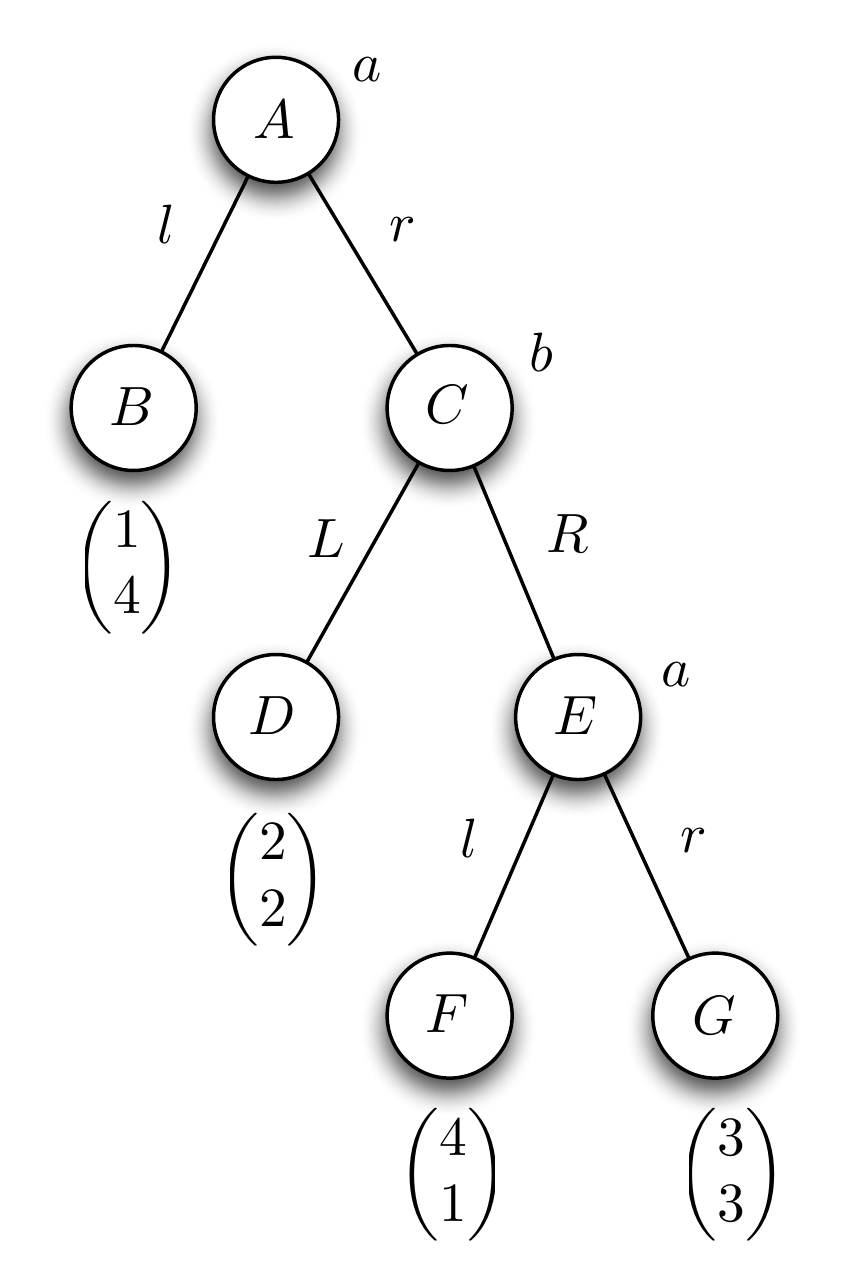}
\end{center}
\caption{A simple extensive form game.}
\label{chap1:fig:extgame}
\end{figure}

This model represents a situation where two
players, $a$ and $b$, take 
turns, with $a$ starting at the
top node $A$. If $a$ plays $l$ (`left') in this node, the game ends in
node $B$ and the payoff for $a$ is $1$ and that for $b$ is $4$.
If $a$, however, plays $r$ in $A$, the game proceeds to node $C$, where it
is $b$'s turn.  Player $b$ has a choice between playing $L$ and $R$
(note that we use upper case to distinguish $b$'s moves from $a$'s moves).
The game continues until a terminal node is reached. We
assume that both players are {\em rational}; that is, each prefers a
higher outcome for themselves over a lower one. What will $a$ play in
the start node $A$?

One way to determine what will happen in this game is to use backward. Consider
node $E$. If that node is reached, given that $a$ is
rational (denoted ${\mathit rat}_a$), $a$ will play $l$
here, since she prefers the outcome $4$ over $3$ (which she would get by
playing $r$). Now consider node $C$. Since $b$ knows that $a$
is rational, he knows that his payoff when playing $R$ at $C$ is
1. Since $b$ is rational, and playing $L$ in $C$ gives him $2$, he
will play $L$. The only thing needed to conclude this is $({\mathit rat}_b
\land K_b{\mathit rat}_a)$. Finally, consider node $A$. Player $a$ can
reason as we just did, so $a$ knows that she has a choice between
the payoff of $2$ she would obtain by playing $r$ and the payoff of $1$ she
would obtain by playing $l$. Since $a$ is rational, she
plays $r$ at $A$. Summarising, the condition that justifies
$a$ playing $r$ at $A$ and $b$ playing $L$ at $B$ is
\[{\mathit rat}_a \land K_a{\mathit rat_b} \land K_aK_b{\mathit rat}_a \land {\mathit rat}_b \land K_b{\mathit rat}_a
\]

This analysis predicts that the game will end in node $D$.  Although
this analysis used only `depth-two' knowledge ($a$ knows that $b$
knows), to perform a similar analysis for longer variants of this game
requires deeper and deeper knowledge of rationality.  In fact, in many epistemic
analyses in game theory, common knowledge of rationality is assumed.
The contribution of epistemic logic to game theory 
is discussed in more detail in
Chapter~\chapref{chap:epistemicfoundationsforgames}.



\subsection{Language}\label{chap1:subsec:language}
Most if not all systems presented in this book extend propositional
logic. The language of propositional logic assumes
a set $\atoms$ of primitive (or atomic) propositions, typically denoted
$p, q, \dots$, possibly with subscripts. They typically refer to
statements that are considered basic; that is, they lack logical
structure, like `it is raining', or `the window is closed'. Classical
logic then uses Boolean operators, such as
$\neg$ (`
not'), $\land$ (`and'), $\lor$, (`or'), $\rightarrow$
(`implies'), and $\leftrightarrow$ (`if and only if'), to build more
complex formulas.  Since all those
operators can be defined in terms of $\land$ and $\neg$ (see
Definition~\ref{1chap:def:abbreviations}), the formal definition of the
language often uses only these two connectives. Formulas
are
denoted with Greek letters: $\varphi, \psi, \alpha, \dots$. So, for
instance, while $(p \land q)$ is the conjunction of two primitive
propositions, the
formula $(\varphi \land \psi)$ is a conjunction of two arbitrary
formulas, each of which may have further structure.

When reasoning about knowledge and belief, we need to be able to refer
to the subject, that is, the agent whose knowledge or belief we are
talking about. To do this, we
assume a finite set $\agents$ of agents. 
Agents are typically denoted $a, b,
\dots, i, j, \dots$, or, in specific examples, $\mathit{Alice},
\mathit{Bob}, \dots$.
To reason about knowledge, we add
operators $K_\agent$ to the language of classical logic, where
$K_\agent\varphi$ denotes `agent $\agent$ knows (or believes)
$\varphi$'.
We typically let the context determine whether $K_\agent$ represents
knowledge or belief. If it is necessary to
reason knowledge and belief simultaneously, we use operators $K_\agent$
for knowledge and $B_\agent$ for belief.
Logics for reasoning about knowledge are sometimes called
\emph{epistemic} logics, while logics for reasoning about belief are
called \emph{doxastic} logics, from the Greek words for knowledge and
belief.
The operators $K_\agent$ and $B_\agent$ are
examples of {\em modal} operators.
We sometimes use $\Box$ or $\Box_\agent$ to denote a generic modal
operator, when we want to discuss general properties of modal operators.

\begin{definition}[An Assemblage of Modal Languages]\label{chap1:def:languages}\index{modal!language}

Let $\atoms$ be a set of primitive propositions, $\operators$ a set of
modal operators, and $\agents$ a set of
agent symbols. Then we define the language $\lan{L}(\atoms,\operators,\agents)$ by the following BNF:
\[
\varphi := p\ \mid\ \neg \varphi\ \mid \ (\varphi \land
\varphi)\ \mid\ \Box\varphi, 
\]
where $p \in \atoms$ and $\Box\in\operators$.
\end{definition}

Typically, the set $\operators$ depends on $\agents$. For instance, the
language for multi-agent epistemic logic is
$\lan{L}(\atoms,\operators,\agents)$, with $\operators = \{K_a \mid a
\in
\agents\}$, that is, we have a knowledge operator for every agent.
To study interactions
between
knowledge and belief, we would have $\operators = \{K_a, B_a \mid a \in
\agents\}$.
The language of propositional logic, which does not involve modal
operators, is denoted $\lan{L}(\atoms)$; \emph{propositional formulas}
are, by definition, formulas in
$\lan{L}(\atoms)$.



\begin{definition}[Abbreviations in the Language]\label{1chap:def:abbreviations}
%
As usual, parentheses are omitted if that does not lead to ambiguity. The
following abbreviations are also standard (in the last one, $\Group
\subseteq \agents$).
\[ \begin{array}{l|l|l}

\text{\it description/name} & \text{\it definiendum} & \text{\it definiens} \\

\hline

\textit{false} & \bot & \atom \et \neg \atom \\
\textit{true} & \top & \neg \bot \\
\text{disjunction} & \phi \vel \psi & \neg (\neg\phi \et \neg\psi) \\
\text{implication} & \phi \imp \psi & \neg \phi \vel \psi \\
\text{dual of $K$} & M_\agent \phi \text{ or }\hat{K}_a \phi & \neg K_\agent \neg \phi \\
\text{everyone in $\Group$ knows} & E_\Group \phi & \Et_{a \in \Group}
K_\agent \phi \\
\end{array} \]
%
Note that $M_a \phi$, which say `agent $a$ does not know $\neg
\phi$', can also be read `agent $a$ considers $\phi$ possible'.
\end{definition}

Let $\Box$ be a modal operator, either one in $\operators$ or one
defined as an abbreviation. We define the $n$th iterated application
of $\Box$, written $\Box^n$, as follows:
\[
\Box^0 \phi = \phi \mbox{ and }  \Box^{n+1} \phi = \Box \Box^n \phi.
\]
%
We are typically interested in iterating the $E_A$ operator, so that we
can talk about `everyone in $A$ knows', `everyone in $A$ knows that
everyone in $A$ knows', and so on.

Finally, we define two measures on formulas. 
\begin{definition}[Length and modal depth]
The {\em length}
$\size{\varphi}$  and the {\em modal depth} $d(\varphi)$ of a formula
$\varphi$ are both defined inductively as follows:
%
\[
\begin{array}{lclclcl}
\size{p} &= &1 &\mbox{and} & d(p) & = & 0 \\
\size{\neg\varphi} & = &\size{\varphi} + 1& \mbox{and} & d(\neg \varphi) & = & d(\varphi) \\
\size{(\varphi \land \psi)} & = & \size{\varphi}+ \size{\psi} +1&\mbox{and} & d(\varphi \land \psi) & = & max\{d(\varphi),d(\psi)\}\\
\size{\Box_a\varphi} & = & \size{\varphi}+ 1& \mbox{and} & d(\Box\varphi)& =& 1 + d(\varphi).
\end{array}\]
%
In the last clause, $\Box_a$ is a modal operator corresponding to a
single agent. Sometimes, if $A
\subseteq \agents$ is a group of agents and $\Box_A$ is a group
operator (like $E_A$, $D_A$ or $C_A$), $\size{\Box_A\varphi}$
depends not only on $\varphi$, but also on the cardinality of $A$.
%
\end{definition}

So, $\size{\Box_a(q \land \Box_bp)} = 5$ and $d(\Box_a(q \land \Box_bp)) = 2$. Likewise, $\size{\Box_aq \land \Box_bp} = 5$ while $d(\Box_aq \land \Box_bp) = 1$.

\subsection{Semantics}\label{chap1:subsec:semantics}
We now define a way to systematically
determine the {\em truth value} of a formula. In propositional logic,
whether $p$ is true or not `depends on the situation'.
The relevant situations are formalised using {\em valuations}, where a
valuation
\[V: \atoms \rightarrow \{\mathit{true}, \mathit{false}\}\]
determines the truth of primitive propositions.  A valuation can be
extended so as to determine the truth of all formulas, using a
straightforward inductive definition:
$\varphi \land \psi$ is true given $V$ iff each of $\varphi$ and $\psi$ is
true given $V$, and
$\neg\varphi$ is true given $V$ iff $\varphi$ is false given
$V$. The truth conditions of disjunctions, implications, and
bi-implications follow directly from these two clauses and
Definition~\ref{1chap:def:abbreviations}.
To model knowledge and belief, we use ideas that go back to Hintikka.
We think of an agent $a$ as considering possible a number of different
situations
that are consistent with the information that the agent has.
Agent $a$ is said to
know (or believe) $\phi$, if $\phi$ is true in all the situations that
$a$ considers possible.  Thus,
rather than using a single situation to give meaning to modal formulas,
we use a {\em set} of such situations; moreover, in each situation, we
consider, for each agent, what other situations he or she considers
possible.  The following example demonstrates how this
is done.

\begin{example}\label{chap1:ex:interviewone}
Bob is invited for a job interview with Alice. They have agreed
that it will take place in a coffeehouse downtown at noon, but the
traffic is quite unpredictable, so it is not guaranteed that either Alice
or Bob
will arrive on time. However, the
coffeehouse is only a 15-minute walk from the bus stop where
Alice plans to go, and a 10-minute walk from the metro station where Bob
plans to go. So, 10 minutes before the interview, both
Alice and Bob will know whether they themselves will arrive on
time. Alice and Bob have never met before.
A Kripke model describing this situation is given in
Figure~\ref{chap1:fig:twoagentmodela}.
\begin{figure}[h!]\center
\begin{center}
\includegraphics[width=0.45\textwidth]{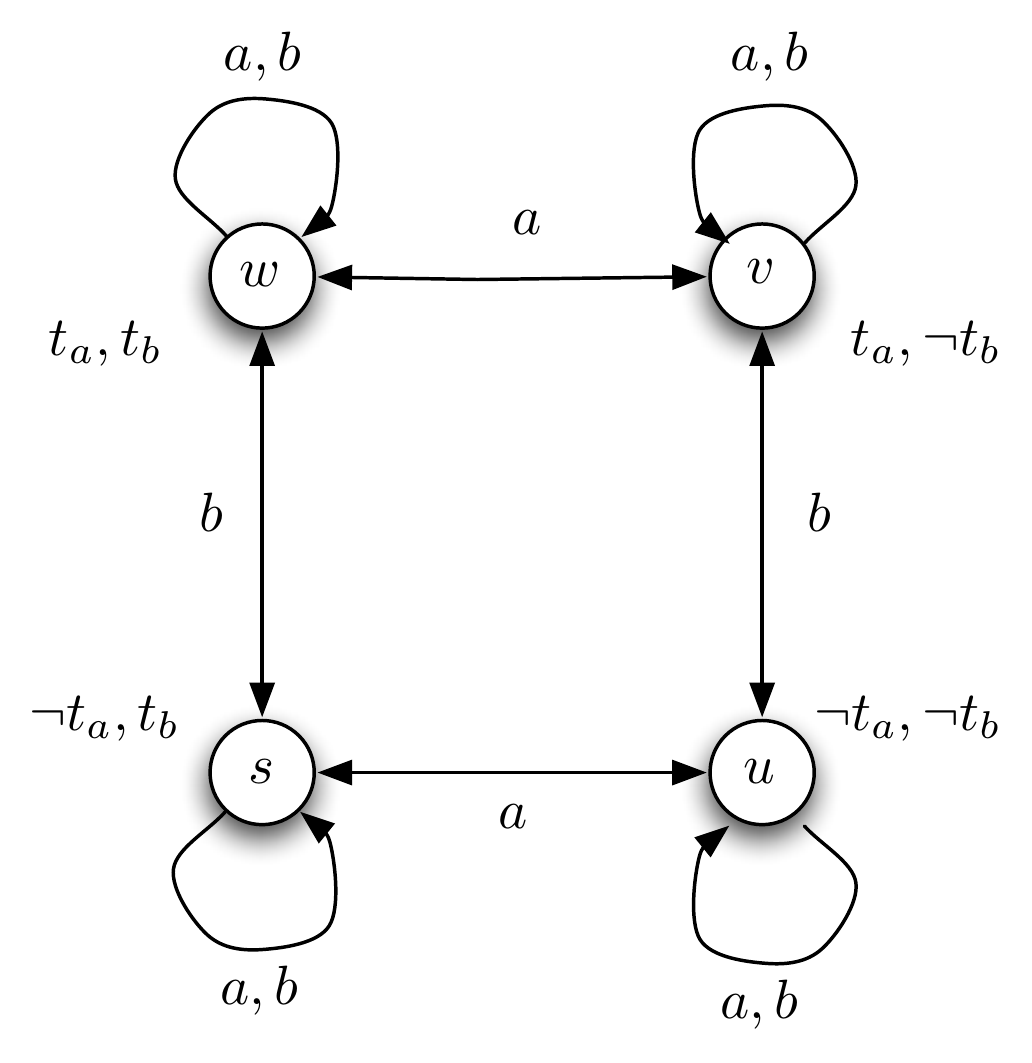}
\end{center}
\caption{The Kripke model for Example~\ref{chap1:ex:interviewone}.}
\label{chap1:fig:twoagentmodela}
\end{figure}

Suppose that at 11:50, both Alice and Bob have just arrived at
their respective stations.  Taking $t_a$ and $t_b$ to represent that
Alice (resp., Bob) arrive on time, this is a situation (denoted $w$ in
Figure~\ref{chap1:fig:twoagentmodela}) where both $t_a$ and $t_b$ are
true. Alice knows that $t_a$ is true (so in $w$ we have
$K_at_a$), but she does not know whether $t_b$ is true; in particular,
Alice considers possible the situation denoted $v$ in
Figure~\ref{chap1:fig:twoagentmodela}, where $t_a
\land \neg t_b$ holds. Similarly, in $w$, Bob considers it possible that
the actual situation is $s$, where Alice is running late but Bob will
make it on time, so that $\neg t_a
\land t_b$ holds. Of course, in $s$, Alice knows that she is late;
that is, $K_a\neg t_a$ holds. Since the only situations that Bob
considers possible at world $w$ are $w$ and
$s$, he knows that he will be on time ($K_bt_b$), and knows that Alice
knows whether or not she is on time ($K_b(K_at_a \lor K_a\neg
t_a)$).  Note that the latter fact follows since $K_a t_a$ holds in
world $w$ and $K_a \neg t_a$ holds in world $s$, so $K_a t_a \lor K_a
\neg t_a$ holds in both worlds that Bob considers possible.
\end{example}

This, in a nutshell, explains what the models for epistemic and doxastic
look like:
they contain a number of situations, typically called {\em states} or
{\em (possible) worlds}, and binary relations on states for each agent,
typically called {\em accessibility relations}.
A pair $(v,w)$ is in the relation for agent $a$ if, in world $v$, agent
$a$ considers state $w$ possible.
Finally, in every state, we need to specify which primitive
propositions are true.

\begin{definition}[Kripke frame, Kripke model]\label{chap1:def:kripkemodel}
Given a set $\atoms$ of primitive propositions and a
set $\agents$ of agents, a {\em Kripke model} is a structure
$M = \langle \States, R^\agents, V^\atoms)$, where
\begin{itemize}
\item
$\States \neq \emptyset$ is a set of states,
sometimes called the {\em domain} of
$M$, and denoted $\Domain(M)$;
\item
$R^\agents$ is a function, yielding an
accessibility relation  $R_a \subseteq \States \times
\States$ for each agent $\agent\in\agents$;
\item
$V^\atoms: S \rightarrow (\atoms \rightarrow \{\mathit{true},\mathit{false}\})$ is a function
that, for all $p\in \atoms$ and $s \in S$, determines what the 
truth value $V^\atoms(s)(p)$ of $p$ is in state $s$
(so $V^\atoms(s)$ is a propositional valuation for each $s \in \States$).
\end{itemize}
We often suppress explicit reference to the sets $\atoms$ and
$\agents$, and write $M = \langle\States, R, V\rangle$, without upper indices.
Further, we sometimes write $s R_a t$ or $R_a s t$ rather than $(s,t) \in R_a$, and use $R_a(s)$ or $R_as$ to denote the set $\{t \in
S \mid R_ast\}$. Finally, we sometimes abuse
terminology and refer to $V$ as a valuation as well.

The class of all Kripke models is denoted $\mathcal{K}$.
We use $\mathcal{K}_m$ to denote the class of Kripke models where
$\size{\agents } = m$.
A {\em Kripke frame} $F = \langle\States,
R\rangle$
focuses on the graph underlying a model, without regard for the
valuation.
\end{definition}

More generally, given a modal logic with a set $\operators$ of modal
operators, the corresponding Kripke model has the form
$M = \langle \States, R^\operators,V^\atoms\rangle$, where there is a
binary relation
$R_\Box$ for every operator $\Box \in \operators$.
$\operators$ may, for example, consist of a knowledge operator for each
agent in some set $\agents$
and a belief operator for each agent in $\agents$.

Given Example~\ref{chap1:ex:interviewone} and
Definition~\ref{chap1:def:kripkemodel}, it should now be clear how the
truth of a formula is determined given a model $M$ and a state $s$. A
pair $(M,s)$ is called a {\em pointed model}; we sometimes drop the
parentheses and write $M,s$.

\begin{definition}[Truth in a Kripke Model]\label{chap1:def:truthdef}
Given a model $M = \langle\States, R^\agents, V^\atoms\rangle$, we
define what it means for a
formula $\varphi$ to be true in $(M,s)$, written $M,s \models
\varphi$, inductively as follows:
\[
\begin{array}{lll}
M,s \models p & \mbox{iff} & V(s)(p) = \mathit{true} {\mbox{ for } p \in
\atoms}\\
M,s \models \varphi \land \psi & \mbox{iff} &
M,s \models \varphi \mbox{ and } M,s \models \psi\\
M,s \models \neg \varphi & \mbox{iff} & \mbox{not } M,s\models
\varphi\  \mbox{ (often written $M,s \not\models\varphi$)}\\
M,s \models K_\agent \varphi & \mbox{iff} & M,t \models \varphi
\mbox{ for all } t
\mbox{ such that } R_\agent st.
\end{array}
\]
More generally, if $M = \langle\States, R^\operators,
V^\atoms\rangle$, then for all $\Box \in \operators$:
\[
M,s \models \Box \varphi  \mbox{ iff }  (M,t) \models \varphi
 \mbox{ for all } t \mbox{ such that } R_\Box st.
\]
%
%
Recall that $M_\agent$ is the dual of $K_\agent$; it easily follows
from the definitions that
$$M,s \models M_\agent \varphi \mbox{ iff there exists some $t$ such
that } R_\agent st \mbox{ and } M,t\models\varphi.$$

%
We write $M \models \phi$ if $M,s \models \phi$ for all $s \in \States$.
\end{definition}
%
%

\begin{example}\label{chap1:ex:interviewtwo}
Consider the model of Figure~\ref{chap1:fig:twoagentmodela}. Note that $K_ap \lor K_a\neg p$ represents the fact that agent $a$ knows {\em whether} $p$ is true. Likewise, $M_ap \land M_a \neg p$ is equivalent to $\neg K_a\neg p \land \neg K_ap$: agent $a$ is ignorant about $p$.
We have the following (in the final items we write $E_{ab}$ instead of $E_{\{a,b\}}$):
\begin{enumerate}
\item $(M,s) \models t_b$:
truth of a primitive  proposition in $s$.
\item $M,s \models (\neg t_a \land K_a\neg t_a \land \neg K_b\neg t_a)
\land (t_b \land \neg K_a t_b \land K_b t_b)$:
at $s$, $a$ knows that $t_a$ is false, but $b$ does not;
similarly, $b$ knows that $t_b$ is true, but $a$ does not.
\item $M \models K_a(K_bt_b \lor K_b\neg t_b) \land K_b(K_at_a \lor
K_a\neg t_a)$:
in all states of $M$, agent $a$ knows that $b$ knows whether $t_b$ is
true, and $b$ knows that $a$ knows whether $t_a$ is true.
\item
$M \models K_a(M_bt_b \land M_b\neg t_b) \land K_b(M_at_a \land M_a\neg t_a)$
in all states of $M$, agent $a$ knows that $b$ does not know whether
$t_a$ is true, and $b$ knows that $a$ does not know whether $t_b$ is
true.
\item $M \models E_{ab}((K_at_a \lor K_a \neg t_a) \land (M_a t_b \land
M_a\neg t_b))$:
in all states, everyone knows that $a$ knows whether
$t_a$ is true, but $a$ does not know whether $t_b$ is true.
\item $M \models E_{ab}E_{ab}((K_at_a \lor K_a \neg t_a) \land (M_a t_b \land
M_a\neg t_b))$:
in all states, everyone knows what we stated in the previous item.
\end{enumerate}
This shows that the model $M$ of Figure~\ref{chap1:fig:twoagentmodela}
is not just a model for a situation where $a$ knows $t_a$ but
not $t_b$ and agent $b$ knows $t_b$ but not $t_a$; it represents much
more information.
\end{example}

As the following example shows, in order to model certain situations, it
may be necessary that some propositional valuations occur in more than
one state in the model.


\begin{example}\label{chap1:ex:interviewthree}
Recall the scenario of the interview between Alice and Bob, as presented
in Example~\ref{chap1:ex:interviewone}. Suppose that we now add the
information that in fact Alice will arrive on time, but Bob is not going
to be on time. Although Bob does not know Alice, he knows that his
friend
Carol is an old friend of Alice. Bob calls Carol, leaving a message on
her machine to ask her to inform Alice about Bob's late arrival as soon
as she is able to do so. Unfortunately for Bob, Carol does not get his
message on time. This situation can be represented in state $M,v$ of
the
model of Figure~\ref{chap1:fig:twoagentmodelb}.

\begin{figure}[h!]\center
\begin{center}
\includegraphics[width=0.75\textwidth]{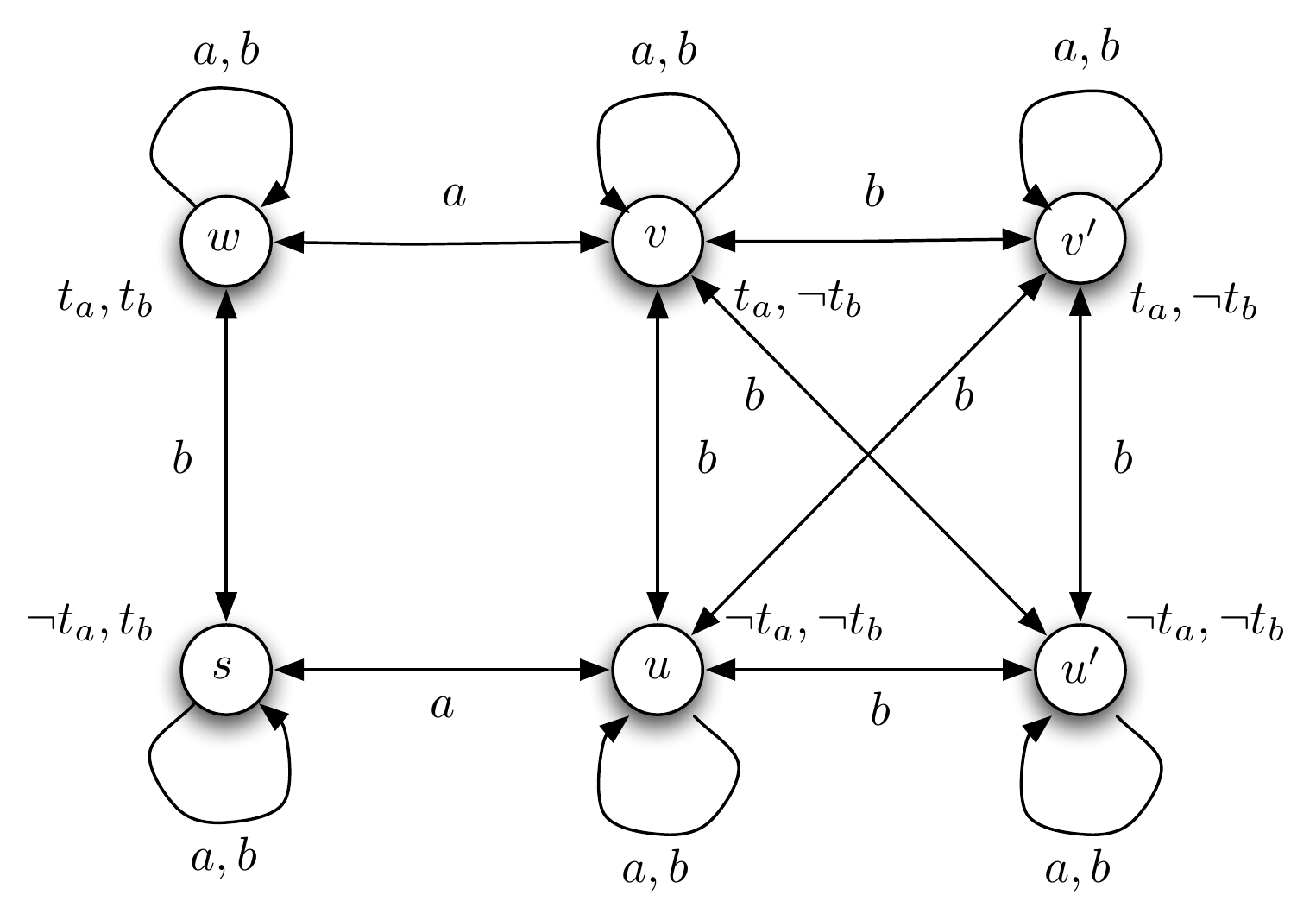}
\end{center}
\caption{The Kripke model for Example~\ref{chap1:ex:interviewthree}.}

\label{chap1:fig:twoagentmodelb}
\end{figure}

Note that in $(M,v)$, we
have $\neg K_a \neg t_b$ (Alice does not know that Bob is late), but
also $M_b(K_a\neg t_b)$ (Bob considers it possible that Alice knows that
Bob is late). So, although the propositional valuations in $v$ and $v'$
are the same, those two states represent different situations: in $v$
agent $a$ is uncertain whether $\neg t_b$ holds, while in $v'$ she knows
$\neg
t_b$. Also, in $M,v$, Bob considers it possible that both of them will
be late, and that Alice knows this: this is because $R_b vu'$ holds in
the model, and $M,u'\models K_a(\neg t_a \land \neg t_b)$.
\end{example}

We often impose restrictions on the accessibility relation.  For
example, we may want to require that if, in world $v$, agent $a$
considers world $w$ possible, then in $w$, agent $a$ should consider $v$
possible.  This requirement would make $R_a$ symmetric.  Similarly, we
might require that, in each world $w$, $a$ considers $w$ itself
possible.  This would make $R_a$ reflexive.  More generally, we are
interested in certain subclasses of models (typically characterized by
properties of the accessibility relations).

\begin{definition}[Classes of models, validity, satisfiability]\label{chap1:def:validity}
Let $\mathcal{X}$ be a class of models, that is, $\mathcal{X} \subseteq
\mathcal{K}$. If
$M\models\varphi$ for all models $M$ in $\mathcal{X}$, we say that {\em
$\varphi$ is valid
in $\mathcal{X}$}, and write $\mathcal{X}\models\varphi$. For example,
for validity in the class of all Kripke models $\mathcal{K}$, we write
$\mathcal{K} \models\varphi$. We write $\mathcal{X}\not\models
\varphi$ when it is not the case that $\mathcal{X} \models \varphi$. So
$\mathcal{X}\not\models\varphi$ holds if, for some model $M \in
\mathcal{X}$ and some $s \in \Domain(M)$, we have $M,s\models
\neg\varphi$.
If there exists a model $M \in \mcx$ and a
state $s \in \Domain({M})$ such that $M,s \models\varphi$, we say that
$\varphi$ is \emph{satisfiable in} $\mcx$.
\end{definition}

We now define a number of classes of models in terms of
properties of the relations $R_a$ in those models. Since they
depend only on the accessibility relation, we
could have defined them for the underlying frames;
indeed, the properties are sometimes called {\em frame
properties}.

\begin{definition}[Frame properties]\label{chap1:def:frameproperties}
Let $R$ be an accessibility relation on a domain of states $\States$.
\begin{enumerate}
\item
$R$ is {\em serial} if for all $s$ there is a $t$ such that $R st$.
The class of serial Kripke models, that is, $\{M = \langle S, R,
V\rangle \mid \mbox{ every }R_a \mbox{ is serial} \}$ is denoted
$\mathcal{KD}$.
\item
$R$ is {\em reflexive} if for all $s$, $R ss$.
The class of reflexive Kripke models  is denoted $\mathcal{KT}$.
\item
$R$ is {\em transitive} if
for all $ s,t,u$, if $R st$ and $R tu$ then $R su$.
The class of transitive Kripke models is denoted
$\mathcal{K}4$.
\item
$R$ is {\em Euclidean} if
for all $s,t,$ and $u$, if $R st$ and $R su$ then  $R tu$.
The class of Euclidean Kripke models is denoted $\mathcal{K}5$
\item
$R$ is {\em symmetric} if
for all $s,t$, if $R st$ then $R ts$.
The class of symmetric Kripke models is denoted $\mathcal{KB}$
\item We can combine properties of relations:
\begin{enumerate}
\item The class of reflexive transitive models is denoted
$\mathcal{S}4$.
\item
The class of transitive Euclidean models is denoted $\mathcal{K}45$.
\item
The class of serial transitive Euclidean models is denoted $\mathcal{KD}45$.
\item
$R$ is an {\em equivalence relation} if $R$ is
reflexive, symmetric, and transitive. It not hard to show that $R$ is
an equivalence relation if $R$ is reflexive and Euclidean.
The class of models  where the relations are equivalence relations is
denoted
$\mathcal{S}5$.

\end{enumerate}
\end{enumerate}
As we did for $\mathcal{K}_m$, we sometimes use the subscript $m$
to denote the number of agents, so $\mathcal{S}5_m$, for instance, is the
class of Kripke models with $\size{\agents } = m$.
\end{definition}

Of special interest in this book is the class $\mathcal{S}5$.
In this case, the accessibility relations are equivalence classes.  This
makes sense if we think of $R_a s t$ holding if $s$ and $t$ are
indistinguishable by agent $a$ based on the information that $a$ has
received.   $\mathcal{S}5$ has typically been used to model knowledge.
In an $\mathcal{S}5$ model,
write $s \sim_a t$ rather than $R_ast$, to emphasize the fact that $R_a$
is an equivalence relation.
When it is clear that $M \in \mathcal{S}5$, when drawing the model, we
omit reflexive arrows, and since the relations are symmetric, we
connect states by a line, rather than using two-way arrows. Finally, we
leave out lines that can be deduced to exist using
transitivity.
We call this the \emph{S5 representation} of a Kripke model.
Figure~\ref{chap1:fig:twoagentmodeld} shows the S5 representation of the
Kripke model of Figure~\ref{chap1:fig:twoagentmodelb}.

\begin{figure}[h!]\center
\begin{center}
\includegraphics[width=0.65\textwidth]{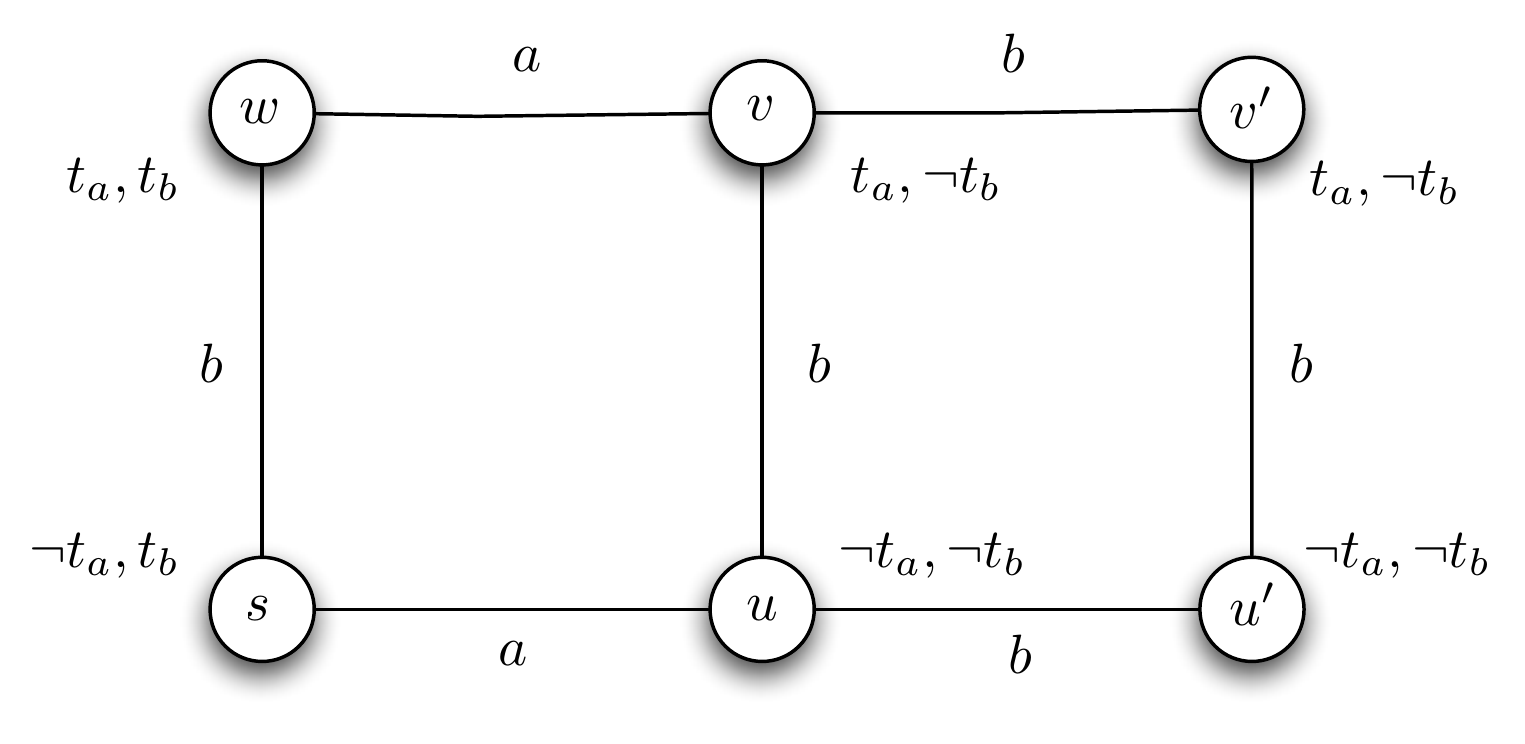}
\end{center}
\caption{The S5 representation of the Kripke model in
Figure~\ref{chap1:fig:twoagentmodelb}.}
\label{chap1:fig:twoagentmodeld}
\end{figure}

When we restrict the classes of models considered, we get some
interesting additional valid formulas.

\begin{theorem}[Valid Formulas]\label{chap1:thm:validities}
Parts (c)--(i) below are valid formulas, where $\alpha$ is a
substitution
instance of a
propositional tautology (see below), $\varphi$ and $\psi$ are arbitrary
formulas, and $\mathcal{X}$ is one of the classes of models defined in
Definition~\ref{chap1:def:frameproperties};
parts (a), (b), and (j) show that we can infer some valid formulas from others.
\begin{enumerate}
\item[(a)]\label{chap1:thm:validities:item:one}
If $\mathcal{X} \models \varphi \rightarrow \psi$ and $\mathcal{X} \models \varphi$, then $\mathcal{X
}\models \psi$.
\item[(b)]\label{chap1:thm:validities:item:two}
If $\mathcal{X} \models \varphi$ then $\mathcal{X} \models K\varphi$.
\item[(c)]\label{chap1:thm:validities:item:three} $\mathcal{X} \models
\alpha$.
\item[(d)]\label{chap1:thm:validities:item:four} $\mathcal{X} \models
K(\varphi \rightarrow \psi) \rightarrow (K\varphi \rightarrow \psi)$.
\item[(e)]\label{chap1:thm:validities:item:five} $\mathcal{KD} \models
K\varphi \rightarrow M\varphi$.
\item[(f)]\label{chap1:thm:validities:item:six}
$\mathcal{T} \models K\varphi \rightarrow \varphi$.
\item[(g)]\label{chap1:thm:validities:item:seven} $\mathcal{K}4 \models
K\varphi \rightarrow KK\varphi$.
\item[(h)]\label{chap1:thm:validities:item:eight} $\mathcal{K}5 \models \neg
K\varphi \rightarrow K\neg K \varphi$.
\item[(i)]\label{chap1:thm:validities:item:nine} $\mathcal{KB} \models \varphi \rightarrow KM\varphi$.
\item[(j)]\label{chap1:thm:validities:item:ten} If $\mathcal{X} \subseteq
\mathcal{Y}$ then $\mathcal{Y} \models \varphi$ implies that
$\mathcal{X}\models\varphi$.
\end{enumerate}
\end{theorem}
Since $\mathcal{S}5$ is the smallest of the classes of models considered
in Definition~\ref{chap1:def:frameproperties}, it easily follows that all
the formulas and inference rules above are valid in $\mathcal{S}5$. To the
extent that we
view $\mathcal{S}5$ as the class of models appropriate for reasoning
about knowledge, Theorem~\ref{chap1:thm:validities} can be viewed as
describing properties of knowledge.  As we shall see, many of these
properties apply to the standard interpretation of belief as well.

Parts~(a) and (c)
emphasise that we represent
knowledge in a logical framework: modus ponens is valid as a reasoning
rule, and we take all propositional tautologies for granted.
In part (c), $\alpha$ is a {\em
substitution instance} of a propositional tautology.  For example, since
$p \lor \neg
p$ and $p \rightarrow (q \rightarrow p)$ are propositional tautologies,
$\alpha$ could be $Kp \lor \neg Kp$ or $K(p \lor q)
\rightarrow (Kr
\rightarrow K(p \lor q))$.
That is, we can substitute an arbitrary formula (uniformly) for a
primitive proposition in a propositional tautology.
Part (b)
says that agents know all valid formulas, and
part (d) says that an agent is able to
apply modus ponens to his own knowledge.
Part (e) is equivalent to $K\varphi
\rightarrow \neg K\neg\varphi$; an agent cannot at the same time
know a proposition and its
negation. Part (f) is even stronger:
it says that what  an agent knows must be
true. Parts (g) and (h)
represent what
has been called {\em positive} and {\em negative introspection},
respectively: an agent
knows what he knows and what he does not
know. Part (i) can be shown to follow
from the other valid formulas; it says that if something is true, the agent
knows that he considers it possible.

\subsubsection{Notions of Group Knowledge}\label{subsec:groupnotions}
So far, all properties that we have encountered are properties of an
individual agent's knowledge.
such as
$E_\Group$,  defined above.  In this section we
introduce two other notions of group knowledge, {\em common knowledge}
$C_\Group$ and {\em distributed knowledge} $D_\Group$, and
investigate their properties.

%
\begin{example}[Everyone knows and distributed
knowledge]\label{chap1:ex:playground}
Alice and Betty each has a daughter;
 their children can each either be at the playground (denoted $p_a$
and $p_b$, respectively)
or at the library ($\neg p_a$, and $\neg p_b$,
respectively). Each child has been carefully instructed that, if
she ends up being on the playground without the other child, she should
call her mother to inform her.
Consider the situation described by the model $M$ in
Figure~\ref{chap1:fig:groupnotions}.
\begin{figure}[h!]\center
\begin{center}
\includegraphics[width=0.55\textwidth]{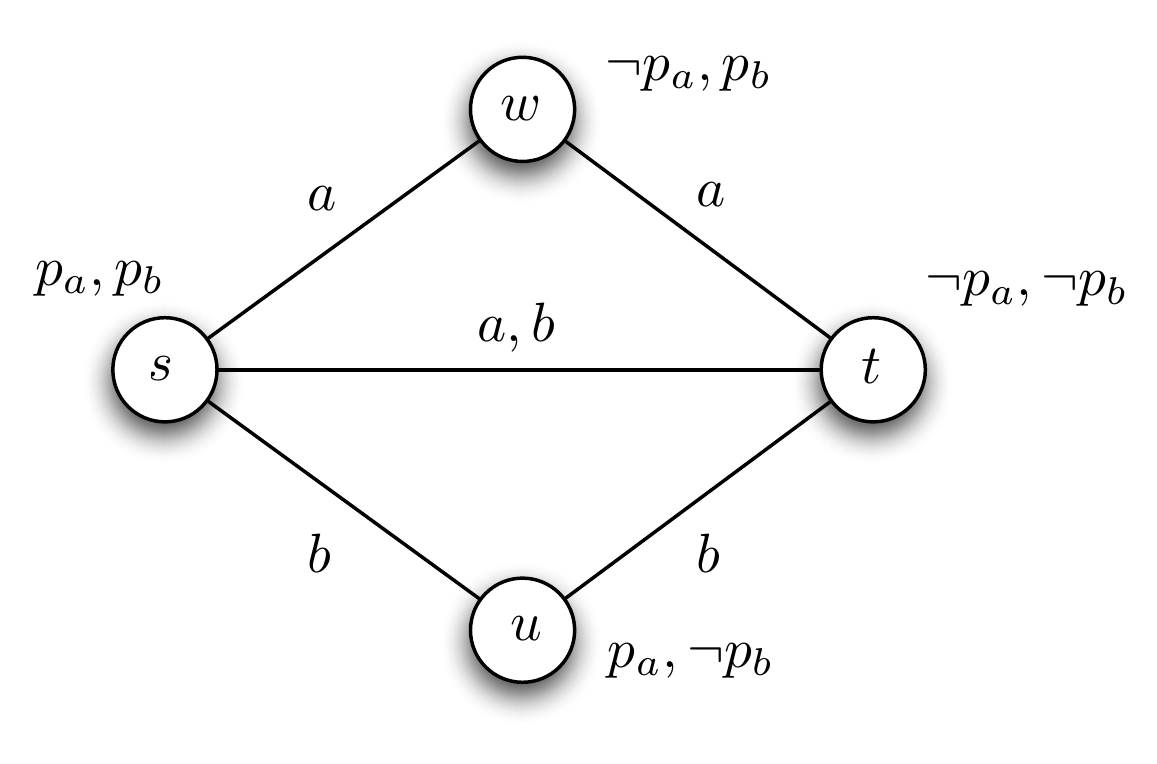}
\end{center}
\caption{The (S5 representation of the) model for
Example~\ref{chap1:ex:playground}.}
\label{chap1:fig:groupnotions}
\end{figure}

We have
$$
M \models ((\neg p_a \land p_b) \leftrightarrow K_a(\neg p_a \land p_b)) \land
 ((p_a \land \neg p_b) \leftrightarrow K_b(p_a \land\neg p_b)).
$$
%
This models the agreement each mother made with her daughter.
%
Now consider the situation at state $s$. We have
$M,s \models K_a\neg(p_a \land \neg p_b)$,
that is, Alice knows that it is not the case that her daughter is
alone at the playground (otherwise her daughter would have informed
her). What does each agent know at $s$? If we consider only
propositional
facts, it is easy to see that Alice knows $p_a \rightarrow p_b$ and
Betty knows $p_b \rightarrow p_a$. What does everyone know at $s$?
%
The following sequence of equivalences is immediate from the definitions:
\[
\begin{array}{ll}
& M,s \models E_{\{a,b\}}\varphi  \\
\mbox{iff} &M,s \models K_a \varphi \land K_b\varphi \\
\mbox{iff} &\forall x(R_asx \Rightarrow M,x\models
\varphi)\ \mbox{ and }\ \forall y(R_bsy \Rightarrow M,y\models
\varphi)\\
\mbox{iff} &\forall x \in \{s, w, t\}\, (M,x\models \varphi) \mbox{
  and } \forall y \in \{s, u, t\}\, (M,y\models \varphi)\\
\mbox{iff} &M \models \varphi.
\end{array}
\]
%

\noindent Thus, in this model, what is known by everyone are just the
formulas valid
in the model.
Of course, this is not true in general.

Now suppose that Alice and Betty an opportunity to talk to each
other. Would they gain any new knowledge? They would indeed.  Since $M,s
\models
K_a(p_a \rightarrow p_b) \land K_b(p_b \rightarrow p_a)$, they would
come to know that $p_a \leftrightarrow p_b$ holds; that is, they would
learn
that their children are at least together, which is certainly not
valid in the model. The knowledge that would emerge if the agents
in a group $A$ were allowed to communicate is  called {\em
distributed knowledge in} $A$, and denoted by the
operator $D_A$.
In our example, we have $M,s \models
D_{\{a,b\}}(p_a \leftrightarrow p_b)$, although $M,s\models
\neg K_a(p_a \leftrightarrow p_b) \land \neg K_b(p_a \leftrightarrow
p_b)$. In other words, distributed knowledge is generally {\em stronger} than
any individual's knowledge, and we therefore cannot define $D_A\varphi$
as $\bigvee_{i\in A}K_i\varphi$, the dual of general knowledge that we
may have expected; that would be weaker than any individual
agent's knowledge. In terms of the model, what would
happen if Alice and Betty could communicate is that Alice could tell
Betty that he should not consider state $u$ possible, while Betty  
could tell Alice that she should not consider state $w$ possible.
So, after communication, the only states considered possible
by both agents at state $s$ are $s$ and $t$. This argument suggests that
we should interpret $D_A$ as the necessity operator ($\Box$-type modal operator) of the relation
$\bigcap_{\agent
\in A}R_\agent$.
By way of contrast, it follows easily from the definitions that $E_A$
can be interpreted as the necessity operator of the relation
$\bigcup_{\agent \in A} R_\agent$.
\end{example}

The following example illustrates common knowledge.

\begin{example}[Common knowledge]\label{chap1:ex:commonknowledge}
This time we have two agents: a sender ($s$) and a receiver ($r$).
If a message is sent,
it is delivered either immediately or with a one-second delay. The
sender sends a message at time $t_0$.  The receiver does not
know that the sender was planning to send the message.
What is each agent's state of knowledge regarding the message?

To reason about this, let $s_z$ (for $z \in \mathbb{Z}$) denote that the
message was sent at time $t_0 + z$, and, likewise, let $d_z$ 
denote that the message was delivered at time $t = z$.
Note that we allow $z$ to be negative.  To see why, consider the world
$w_{0,0}$ where the message arrives
immediately (at time $t_0$). 
(In general, in the subscript $(i,j)$ of a world $w_{i,j}$, $i$ denotes
 the time that the message was sent, and $j$ denotes the time it was received.)
In world $w_{0,0}$, the receiver considers it
possible
that the message was sent at time $t_0 - 1$.  That is, the receiver
considers possible the world $w_{-1,0}$ where the message was sent at $t_0-1$
and took one second to arrive.  In world $w_{-1,0}$, the sender
considers
possible the world $w_{-1,-1}$ where the message was sent at time $t_0
-1$ and arrived immediately.  And in world $w_{-1,-1}$, the receiver
considers
possible a world $w_{-2,-1}$ where the message as sent at time $t_0-2$.
(In general, in world $w_{n,m}$, the message is sent at time $t_0+n$ and
received at time $t_0 + m$.)  In addition, in world $w_{0,0}$, the
sender considers possible world $w_{0,1}$, where the message is received
at time $t_0 +1$.
The situation is described in the following model $M$.
\begin{figure}[h!]\center
\begin{center}
\includegraphics[width=1\textwidth]{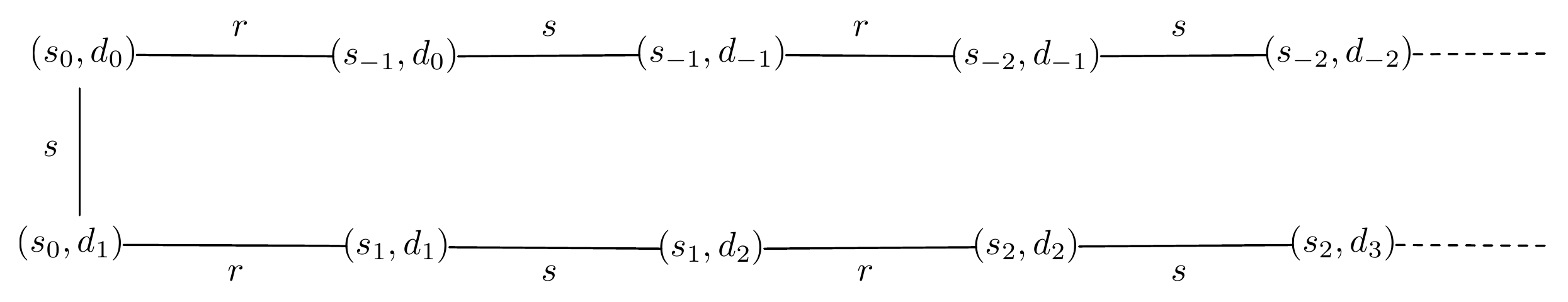}
\end{center}
\caption{The (S5 representation of the) model for
Example~\ref{chap1:ex:commonknowledge}.}
\label{chap1:fig:commonknowledge}
\end{figure}

Writing $E$ for `the sender and receiver both know', it easily follows
that
$$M,w_{0,0} \models s_0 \land d_0 \land \neg E \neg s_{-1}
\land \neg E \neg d_1 \land \neg E^3 \neg s_{-2}.$$

The notion of $\phi$ being {\em common knowledge} among  group $A$,
denoted $C_A\varphi$, is meant to capture the idea that, for all $n$,
$E^n\varphi$ is true. Thus, $\phi$ is {\em not} common
among $A$ if someone in $A$ considers it possible that
someone in $A$ considers it possible that \ldots someone in $A$
considers it possible that $\phi$ is false. 
This is formalised below, but the reader should already be
convinced that in our scenario, even if it is common knowledge among
the agents that messages will have either no delay or a one-second
delay, it is {\em not}
common knowledge that the message was sent at or after time $t_0 - m$
for any value of $m$!
\end{example}

\begin{definition}[Semantics of three notions of group knowledge]
\label{chap01:def:groupnotions}
Let $\group \subseteq \agents$ be a group of agents. Let $R_{E_\group} =
\cup_{a \in\group} R_a$.
%
As we observed above,
$$
(M,s) \models E_\group \varphi \mbox{ iff for all $t$ such that }
R_{E_\group}st, \mbox{ we have } (M,t) \models \varphi.
$$
Similarly, taking $R_{D_\group} = \cap_{a \in \group} R_a$, we have
$$
(M,s) \models D_\group \varphi \mbox{ iff for all $t$ such that }
R_{D_\group}st, \mbox{ we have } (M,t) \models \varphi.
$$

Finally, recall that the {\em transitive closure} of a relation $R$ is
the smallest relation $R^+$ such that $R
\subseteq R^+$, and such that, for all $x, y,$ and $z$, if $R^+xy$ and
$R^+yz$ then $R^+xz$.
We define $R_{C_\Group}$ as $R_{E_\group}^+ =
(\bigcup_{\agent\in\Group} R_\agent)^+$. Note that, in
Figure~\ref{chap1:fig:commonknowledge}, {\em every} pair of states is
in the relation $R_{C_{\{r,s\}}}^+$. In general, we have
$R_{C_\group}st$ iff there is some path $s = s_0, s_1, \dots, s_n = t$
%
from $s$ to $t$ such that $n\ge 1$ and, for all $i < n$, there is some
agent $a\in\group$ for which $R_as_is_{i+1}$. Define
$$
(M,s) \models C_\group \varphi \mbox{ iff for all $t$ such that }
R_{C_\group} st, (M,t) \models \varphi.
$$
%
\end{definition}

%
It is almost immediate from the definitions that, for $a \in A$, we have
\begin{equation}\label{chap1:eq:groupknowledge}
\mathcal{K} \models (C_\Group\varphi \rightarrow E_\Group\varphi)
\land (E_\Group\varphi \rightarrow  K_\agenta\varphi) \land
(K_\agent\varphi \rightarrow D_\Group\varphi).
\end{equation}
%
Moreover, for $\mathcal{T}$ (and hence also for $\mathcal{S}4$ and
$\mathcal{S}5$), we have
$$\mathcal{T} \models D_a \phi \rightarrow \phi.$$

The relative strengths shown in \eqref{chap1:eq:groupknowledge} are strict in the
sense that none of the converse implications are valid (assuming that $A
\neq
\{a\}$).

%
We conclude this section by defining some languages that are used later
in this chapter.  Fixing $\atoms$ and $\agents$, we write $\lan{L}_X$ for
the language $\lan{L}(\atoms,\operators,\agents)$, where
\[
\begin{array}{ll}
X = K & \mbox{if } \operators = \{K_a \mid a \in \agents\}\\
X = CK & \mbox{if } \operators = \{K_a, C_A \mid a \in \agents, A \subseteq \agents\}\\
X = DK & \mbox{if } \operators = \{K_a, D_A \mid a \in \agents, A \subseteq \agents\}\\
X = CDK & \mbox{if } \operators = \{K_a, C_A, D_A \mid a \in
\agents, A \subseteq \agents\}\\
X = EK & \mbox{if } \operators = \{K_a, E_A \mid a \in
\agents, A \subseteq \agents\}.\\
\end{array}
\]

\subsubsection{Bisimulation}\index{bisimulation|(}
It may well be that two models $(M,s)$ and $(M',s')$ `appear
different', but still satisfy the same formulas. For example,
consider the  models $(M,s)$, $(M',s')$, and $(N,s_1)$ in
Figure~\ref{chap1:fig:similarmodels}.
As we now show, they satisfy the same formulas.
We actually prove something even stronger.  We show that
all of $(M,s)$, $(M,t)$,
$(M',s')$, $(N,s_1)$, $(M,s_2)$, and $(N,s_3)$ satisfy the same
formulas, as do all 
of $(M,u)$, $(M,w)$, $(M',w')$, $(N,w_1)$, and $(N,w_2)$.  For the
purposes of the proof, call the models in the first group {\em green},
and the models in the second group {\em red}.
We now show, by induction on the structure of formulas, that all green
models satisfy the same formulas, as do all red models.
For primitive propositions, this is immediate. And if two
models of the same colour agree on two formulas, they also agree on
their negations and their conjunctions. The other formulas we need to
consider are knowledge formulas. Informally, the argument is this. Every
agent considers, in every pointed model, both green and red models
possible. So his knowledge in each pointed model is the same.
We now formalise this reasoning.


\begin{definition}[Bisimulation] \label{del.exp.bisim.bisimdef}
Given models $M= (\States ,R,V)$ and $M'=( \States',R',V' )$,
a non-empty relation $\bisrel \subseteq \States
\times \States'$\index{R@$\bisrel$} is a {\em bisimulation between $M$
and $M'$} iff for all
$\state
\in \States$ and $\state' \in \States'$ with $(\state,\state') \in
\bisrel$:

\begin{itemize}
\item
$V(\state)(p)=V'(\state')(p)$ for all $p \in
\atoms$;
\item
for all $\agent \in \agents$ and all $\statet \in \States$, if
$R_\agent\state\statet$, then there is a $\statet'\in
\States'$ such that $R'_\agent\state'\statet'$ and
$(\statet,\statet') \in \bisrel$;
\item
for all $\agent \in \agents$ and all $\statet' \in \States'$,
if $R'_\agent\state'\statet'$, then there is a
$\statet \in \States$ such that $R_\agent\state\statet$
and $(\statet,\statet') \in \bisrel$.
\end{itemize}

We write $(M,\state) \bisim (M',\state')$ iff there is a bisimulation
between $M$ and $M'$ linking $\state$ and $\state'$. If so, we
call
$(M,\state)$ and $(M',\state')$ {\em bisimilar}.
\end{definition}

Figure~\ref{chap1:fig:similarmodels} illustrates some bisimilar models.
\begin{figure}[h!]\center
\begin{center}
\includegraphics[width=0.7\textwidth]{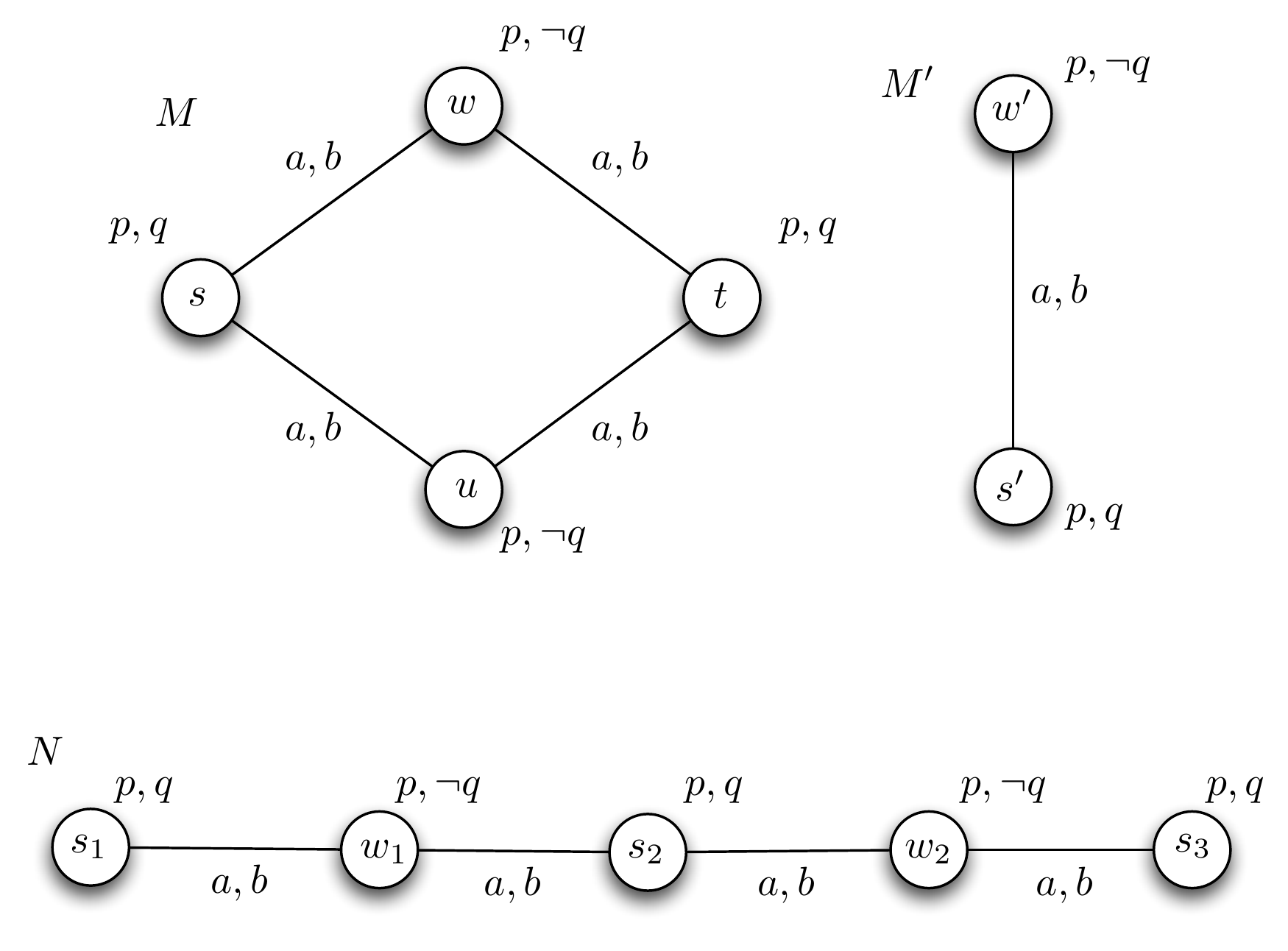}
\end{center}
\caption{Bisimilar models.}
\label{chap1:fig:similarmodels}
\end{figure}
In terms of the models of Figure~\ref{chap1:fig:similarmodels}, we have
$M,s \bisim M',s'$,  $M,s \bisim N,s_1$, etc.  We are interested in
bisimilarity because, as the
following theorem shows, bisimilar models satisfy the same formulas involving the operators $K_\agent$ and $C_\group$.

\begin{theorem}[Preservation under bisimulation]\label{chap1:thm:bisim}

Suppose that $(M,s) \bisim (M',s')$. Then, for all
formulas $\phi \in \lan{L}_{CK}$, we have
\[ M,s\models \varphi\Leftrightarrow M',s'\models \varphi.
\]
\end{theorem}
The proof of the theorem proceeds by induction on the structure of 
formulas, much as in our example.
We leave the details to the reader.

Note that Theorem~\ref{chap1:thm:bisim} does not claim that distributed knowledge is preserved under bisimulation, and indeed, it is not, i.e., Theorem~\ref{chap1:thm:bisim} does not hold for a language with  $D_\group$ as an operator.
Figure~\ref{chap1:fig:distknowledge} provides a witness for this.   We leave it to the reader to check that
although
$(M,s) \bisim (N,s_1)$ for the two pointed models of
Figure~\ref{chap1:fig:distknowledge}, we nevertheless have $(M,s)
\models \neg
D_{\{a,b\}} p$ and $(N,s_1) \models D_{\{a,b\}}p$.

\begin{figure}[h!]\center
\begin{center}
\includegraphics[width=0.8\textwidth]{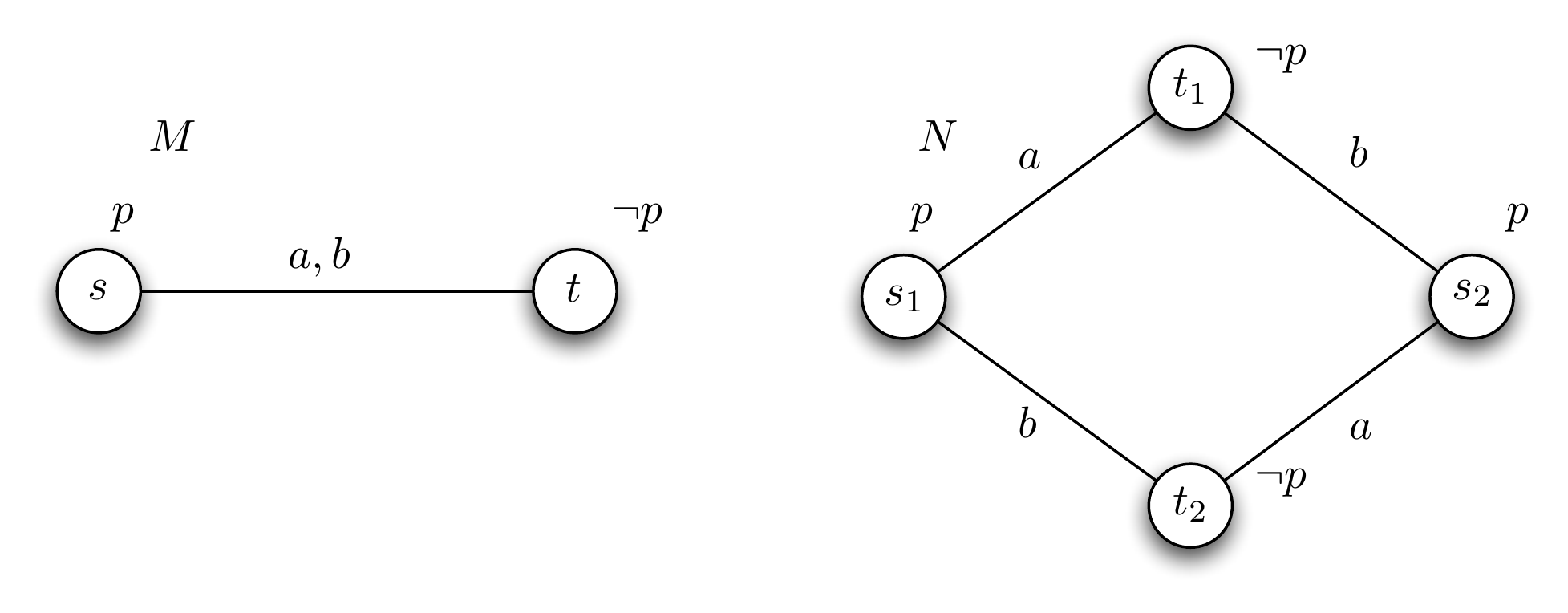}
\end{center}
\caption{Two bisimilar models that do not preserve distributed
knowledge.}
\label{chap1:fig:distknowledge}
\end{figure}
\index{bisimulation|)}

We can, however, generalise the notion of bisimulation to that of a {\em
  group bisimulation} and `recover' the preservation theorem, as
follows. If $\group\subseteq\agents$,
$\state$ and $\statet$ are states, then we write $R_\group\state\statet$ if
$A=\{a \mid R_a\state\statet\}$. That is,
$R_\group\state\statet$ holds if the set of agents $\agent$
for which $\state$ and $\statet$ are $\agent$-connected is exactly
$\group$. $(M,s)$ and $(M',s')$ are \emph{group bisimilar},
written $(M,s) \bisim_{\mathit group} (M',s')$, if the conditions of
Definition~\ref{del.exp.bisim.bisimdef} are met when every
occurrence of an individual agent $a$ is replaced by the group $A$.
Obviously, being group bisimilar implies being bisimilar. Note that
the models $(M,s)$ and $(N,s_1)$ of
Figure~\ref{chap1:fig:distknowledge} are bisimilar, but not group
bisimilar. The proof of Theorem~\ref{chap1:thm:groupbisim} is
analogous to that of Theorem~\ref{chap1:thm:bisim}.

\begin{theorem}[Preservation under bisimulation]\label{chap1:thm:groupbisim}
Suppose that $(M,s) \bisim_{\mathit group} (M',s')$. Then, for all
formulas $\phi \in \lan{L}_{CDK}$, we have
\[ M,s\models \varphi\Leftrightarrow M',s'\models \varphi.
\]
\end{theorem}

\subsection{Expressivity and Succinctness}
If a number of formal languages can be used to
model similar phenomena, a natural question to ask is which language
is `best'. Of course, the answer depends on how `best' is measured.
In the next section, we compare various languages in terms of
the computational complexity of some reasoning
problems.  Here, we consider the
notions of {\em expressivity} (what can be expressed in the language?)
and {\em succinctness} (how economically can one say it?).

\subsubsection{Expressivity}

To give an example of expressivity and the tools that are used to study
it, we start by showing that finiteness of models cannot be expressed in
epistemic logic, even if the language includes operators for common
knowledge and distributed knowledge.

\begin{theorem}\label{thm:finite}
There is no formula $\phi\in\lan{L}_{CDK}$ such that, for all
$\mathcal{S}5$-models $M = \langle S, R, V\rangle$,
\[
M \models \varphi \mbox{ iff } S \mbox{ is finite}
\]
\end{theorem}
\begin{proof}
  Consider the two models $M$ and $M'$ of Figure~\ref{fig:finite}.
\begin{figure}[h]\center
\begin{center}
\includegraphics[width=0.7\textwidth]{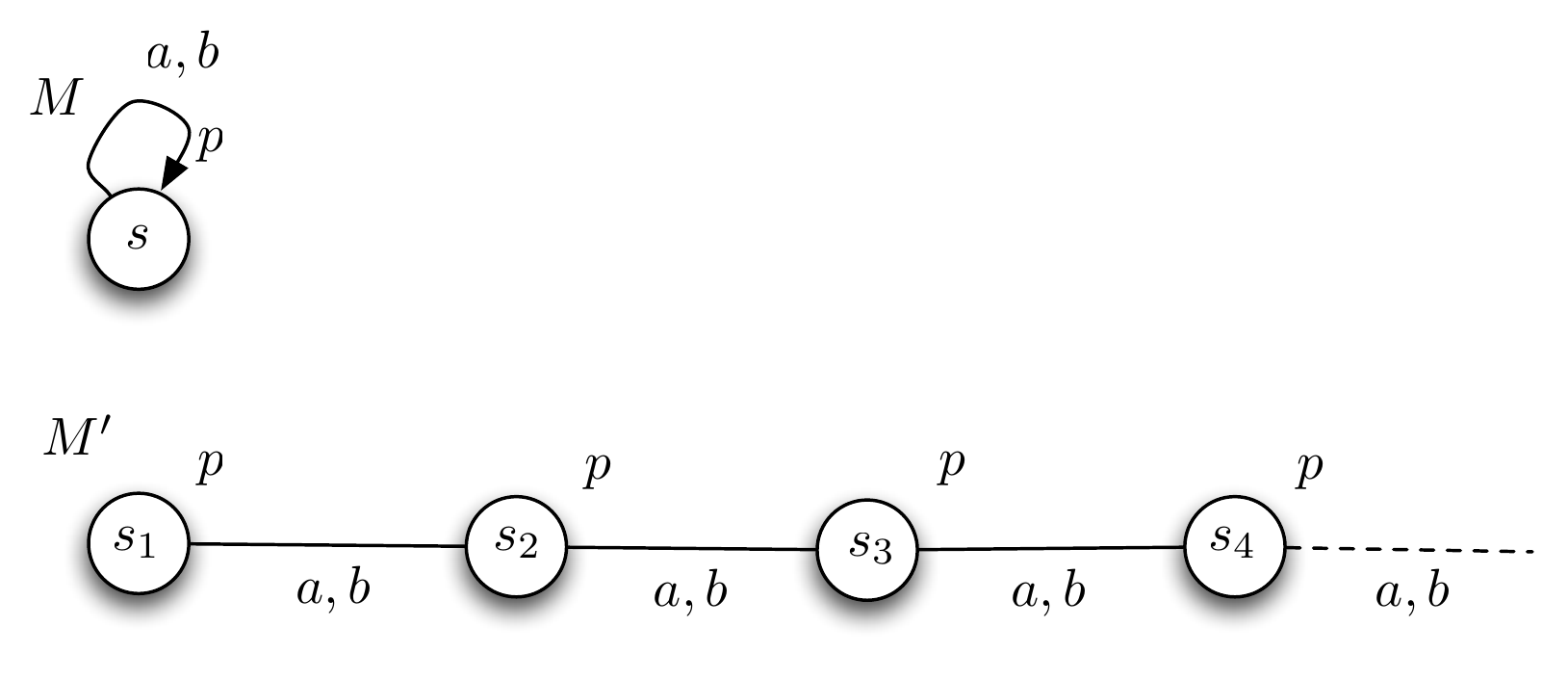}
\caption{A finite and an infinite model where the same formulas
are valid.}\label{fig:finite}
\end{center}
\end{figure}
%
  Obviously, $M$ is finite and $M'$ is not. Nevertheless, the two models are
easily seen to be group bisimilar, so they cannot be distinguished by epistemic
  formulas. More precisely, for all formulas $\varphi \in
  \lan{L}_{CDK}$, we have $M,s \models \varphi$ iff $M', s_1 \models
  \varphi$ iff $M', s_2 \models \varphi$ iff $M', s_n \models \varphi$
  for some $n \in \Nat$, and hence $M \models \varphi$ iff $M' \models
  \varphi$.
\end{proof}


It follows immediately from Theorem~\ref{thm:finite} that finiteness
cannot be expressed in the language $\lan{L}_{CDK}$ in a class ${\mathcal
X}$ of models containing ${\mathcal S}5$.

We next prove some results that let us compare the expressivity of two
different languages.  We first need some definitions.

\begin{definition}
Given a class ${\mathcal X}$ of models, formulas $\varphi_1$ and $\varphi_2$
are \emph{equivalent on} ${\mathcal X}$, written $\varphi_1 \equiv_{\mathcal X} \varphi_2$,
  if, for all $(M,s) \in \mathcal{X}$, we have that $M,s \models
  \varphi_1$ iff $M,s \models \varphi_2$.
Language $\lan{L}_2$ is {\em at least as expressive as}
$\lan{L}_1$ on ${\mathcal X}$, written $\lan{L}_1 \sqsubseteq_{\mathcal
X}\lan{L}_2$ if, for every formula $\varphi_1\in\lan{L}_1$, there is a formula
$\varphi_2
\in \lan{L}_2$ such that $\varphi_1 \equiv_{\mathcal X} \varphi_2$.
$\lan{L}_1$ and $\lan{L}_2$ are {\em equally expressive} on ${\mathcal X}$ if
$\lan{L}_1\sqsubseteq_{\mathcal X}\lan{L}_2$ and
$\lan{L}_2\sqsubseteq_{\mathcal X}\lan{L}_1$. If
$\lan{L}_1\sqsubseteq_{\mathcal X}\lan{L}_2$  but
$\lan{L}_2 \not\sqsubseteq_{\mathcal X} \lan{L}_1$, then
$\lan{L}_2$ is {\em more expressive than} $\lan{L}_1$ on
${\mathcal X}$, written $\lan{L}_1 \sqsubset_{\mathcal X}\lan{L}_2$.
\end{definition}

Note that if ${\mathcal Y} \subseteq {\mathcal X}$,
then  $\lan{L}_1 \sqsubseteq_{\mathcal X}\lan{L}_2$ implies
$\lan{L}_1 \sqsubseteq_{\mathcal Y} \lan{L}_2$, while $\lan{L}_1
\not\sqsubseteq_{\mathcal Y}\lan{L}_2$ implies $\lan{L}_1
\not\sqsubseteq_{\mathcal X}\lan{L}_2$.
Thus, the strongest results that we can show for the classes of models
of interest to us are $\lan{L}_1 \sqsubseteq_{\mathcal K} \lan{L}_2$ and
$\lan{L}_1 \not\sqsubseteq_{\mathcal{S}5}\lan{L}_2$

With these definitions in hand, we can now make precise that common
knowledge `really adds' something to epistemic
logic.

\begin{theorem}\label{thm:firstexpress}
$\lan{L}_K \sqsubseteq_{\mathcal{K}} \lan{L}_{CK}$ and
$\lan{L}_K \not\sqsubseteq_{\mathcal{S}5} \lan{L}_{CK}$.
\end{theorem}
\begin{proof}
Since $\lan{L}_K \subseteq \lan{L}_{CK}$, it is obvious that $\lan{L}_K
\sqsubseteq_{\mathcal{K}} \lan{L}_{CK}$.  To show that $\lan{L}_{CK}
\not\sqsubseteq_{\mathcal{S}5}
\lan{L}_{K}$,  consider the sets of pointed models
${\mathcal M} = \{(M_n,s_1)\mid n\in\Nat\}$ and ${\mathcal N}
=\{(N_n,t_1)\mid n \in \Nat\}$ shown in Figure~\ref{fig:express
common}. The two models $M_n$ and $N_n$ differ only in $(M_n,s_{n+1})$
(where $p$ is false) and
$(N_n,t_{n+1})$ (where $p$ is true). In
particular, the first $n-1$ states of $(M_n,s_1)$ and $(N_n,t_1)$
are the same. As a consequence, it is easy to show that,
\begin{equation}\label{eq:depth}
\mbox{for all }   n \in \Nat \mbox{ and } \varphi \in \lan{L}_K \mbox{
with } d(\varphi) < n, \ (M_n,s_1) \models \varphi \mbox{ iff }
(N_n,t_1)\models\varphi.
\end{equation}
%
Clearly ${\mathcal M} \models C_{\{a,b\}} \neg p$ while
${\mathcal N} \models \neg C_{\{a, b\}} \neg p$. If there were a
formula $\varphi \in \lan{L}_K$  equivalent to $C_{\{a,b\}}\neg
p$, then we would have ${\mathcal M} \models \varphi$ while ${\mathcal
N} \models \neg\varphi$.  Let $d = d(\varphi)$,
and consider the pointed models $(M_{d+1},s_1)$ and
$(N_{d+1},t_1)$. Since the first is a member of ${\mathcal M}$ and the
second of ${\mathcal N}$, the pointed models disagree on $C_{\{a,b\}}
\neg p$;
however, by \eqref{eq:depth}, they agree on $\varphi$. This is obviously
a contradiction, therefore a formula $\varphi \in \lan{L}$ that is equivalent to
$C_{\{a,b\}}\neg p$ does not exist.

\begin{figure}[h!]\center
\begin{center}
\includegraphics[width=0.9\textwidth]{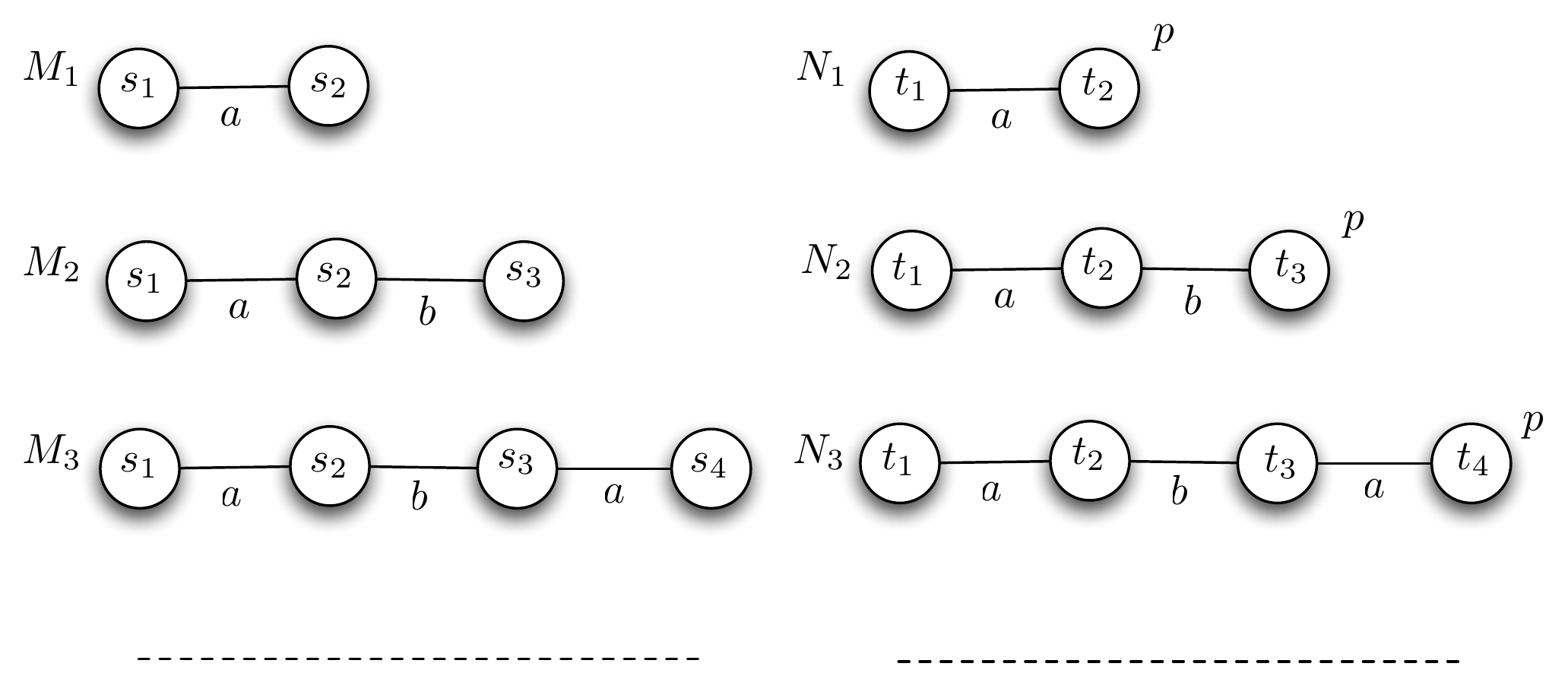}
\caption{Models $M_n$ and $N_n$. The atom $p$ is only true in the pointed models $(N_n,s_{n+1})$.}\label{fig:express common}
\end{center}
\end{figure}
\end{proof}


The next result shows, roughly speaking, that distributed knowledge is
not expressible using knowledge and common knowledge, and that common
knowledge is not expressible using knowledge and distributed knowledge.

\begin{theorem}\label{thm:expressiveness}
\begin{enumerate}
\item[(a)] $\lan{L}_K \sqsubseteq_{\mathcal{K}} \lan{L}_{DK}$ and
$\lan{L}_K \not\sqsubseteq_{\mathcal{S}5} \lan{L}_{DK}$;
\item[(b)] $\lan{L}_{CK} \not\sqsubseteq_{\mathcal{S}5} \lan{L}_{DK}$;
\item[(c)] $\lan{L}_{DK} \not\sqsubseteq_{\mathcal{S}5} \lan{L}_{CK}$;
\item[(d)] $\lan{L}_{CK} \sqsubseteq_{\mathcal{K}} \lan{L}_{CDK}$ and
$\lan{L}_{CDK} \not\sqsubseteq_{\mathcal{S}5} \lan{L}_{CK}$;
\item[(e)] $\lan{L}_{DK} \sqsubseteq_{\mathcal{K}} \lan{L}_{CDK}$  and
$\lan{L}_{CDK} \not\sqsubseteq_{\mathcal{S}5} \lan{L}_{DK}$.
\end{enumerate}
\end{theorem}
\begin{proof}
For part (a),  $\sqsubseteq_{\mathcal{K}}$ holds trivially.  We use the
models in Figure~\ref{chap1:fig:distknowledge} to show that
$\lan{L}_{DK} \not\sqsubseteq_{\mathcal{S}5} \lan{L}_K$.
Since $(M,s)
\bisim (N,s_1)$,
the models verify the same $\lan{L}$-formulas. However, $\lan{L}_{DK}$
discriminates them: we have $(M,s) \models \neg D_{\{a,b\}}p$, while
$(N,s_1) \models D_{\{a,b\}} p$.  Since $(M,s)$ and $(N,s_1)$ also
verify the same $\lan{L}_{CK}$-formulas, part (3) also follows.

For part (b), observe that \eqref{eq:depth} is also true for all formulas
$\varphi \in \lan{L}_{DK}$, so the formula $C_{\{a,b\}}\neg p \in
\lan{L}_{CK}$ is not equivalent to a formula in $\lan{L}_{DK}$.

Part (c) is proved using exactly the same models and
argument as part (a).

For part (d), $\sqsubseteq$ is
obvious. To show that $\lan{L}_{CDK} \not\sqsubseteq_{{\mathcal S}5}
\lan{L}_{DK}$, we can use the models and argument of part (b).
Similarly, for part (e), $\sqsubseteq$ is
obvious. To show that $\lan{L}_{CDK} \not\sqsubseteq_{{\mathcal S}5}
\lan{L}_{DK}$, we can use the models and argument of part (a).
\end{proof}

We conclude this discussion with a remark about distributed knowledge. We
informally described distributed knowledge in a group as the knowledge
that would obtain were the agents in that group able to
communicate. However, Figure~\ref{chap1:fig:distknowledge} shows that
this intuition is not quite right. First, observe that both
$a$ and $b$ know the same formulas in $(M,s)$ and $(N,s_1)$; they even
know the same formulas in $(M,s)$ and $(N,s_1)$. That is, for all
$\varphi\in \lan{L}_K$, we have
\[(M,s) \models K_a \varphi \mbox{ iff } (M,s) \models K_b\varphi \mbox{ iff } (N,s_1) \models K_a \varphi \mbox{ iff } (N,s_1) \models K_b\varphi\]

But if both agents possess the same knowledge in $(N,s_1)$, how can
communication help them in any way, that is, how can it be that there is
distributed knowledge (of $p$) that no individual agent has?
Similarly, if $a$ has the same knowledge in $(M,s)$ in $(N,s_1)$, and so does
$b$, why would communication in one model ($N$) lead them to know
$p$, while in the other, it does not? Semantically, one could argue that
in $s_1$ agent $a$ could `tell' agent $b$ that $t_2$ `is not possible',
and $b$ could `tell' $a$ that $t_1$ `is not possible'. But how would
verify the same formulas? This observation has led some researchers to
require that distributed knowledge be interpreted in what are
called {\em
bisimulation contracted models} (see the notes at the end of the chapter for
references). Roughly, a model is bisimulation contracted if it does not
contain two points that are bisimilar. Model $M$ of
Figure~\ref{chap1:fig:distknowledge} is bisimulation contracted, model
$N$ is not.

\subsubsection{Succinctness}
Now suppose that two languages $\lan{L}_1$ and $\lan{L}_2$ are equally
expressive on ${\mathcal X}$, and also that their computational
complexity of the reasoning problems for them is equally good, or
equally bad. Could we still prefer one language over the other? {\em
Representational succinctness} may provide an answer here: it may be the
case that the description of some properties is much shorter in
one language than in the other.

But what does `much shorter' mean?  The fact that there is a
formula $\lan{L}_1$ whose length is 100 characters less than the shortest
equivalent formula in $\lan{L}_2$ (with respect to some class
$\mathcal{X}$ of models) does not by
itself make $\lan{L}_1$ much more succinct that $\lan{L}_2$. 

We want to capture the idea that $\lan{L}_1$ is exponentially more
succinct than $\lan{L}_2$.  We cannot do this by looking at just one
formula.  Rather, we need a sequence of formulas $\alpha_1, \alpha_2,
\alpha_3, \ldots$ in $\lan{L}_1$, where the gap in size between $\alpha_n$ and
the shortest formula equivalent to $\alpha_n$ in $\lan{L}_2$ grows
exponentially in $n$.  This is formalised in the next definition.

\begin{definition}[Exponentially more succinct]
Given a class ${\mathcal X}$ of models, $\lan{L}_1$ is
\emph{exponentially more succinct} than $\lan{L}_2$ on ${\mathcal X}$ if
the following conditions hold:
\begin{enumerate}
\item[(a)] for every formula $\beta \in \lan{L}_2$, there is a formula
$\alpha \in \lan{L}_1$ such that $\alpha \equiv_{\mathcal{X}} \beta$ and
$\size{\alpha} \le \size{\beta}$.
\item[(b)] there exist $k_1, k_2 > 0$, a sequence $\alpha_1, \alpha_2,
\ldots$ of formulas in $\lan{L}_1$, and a sequence $\beta_1, \beta_2,
\ldots$ of formulas in $\lan{L}_2$ such that, for all $n$, we have:
\begin{itemize}
\item[(i)] $\size{\alpha_n} \le k_1 n$;
\item[(ii)]  $\size{ \beta_n} \ \geq 2^{k_2n}$;
\item[(iii)] $\beta_n$ is the shortest formula in $\lan{L}_2$ that is
equivalent to $ \alpha_n$ on ${\mathcal X}$.
\end{itemize}
\end{enumerate}
\end{definition}

In words, $\lan{L}_1$ is exponentially more succinct than $\lan{L}_2$ if,
for every formula $\beta \in \lan{L}_2$, there is a formula in
$\lan{L}_1$ that is equivalent and no longer than $\beta$, but there is
a sequence $\alpha_1, \alpha_2, \ldots$ of formulas in $\lan{L}_1$ whose
length increases at most linearly, but there is no sequence $\beta_1,
\beta_2, \ldots$ of formulas in $\lan{L}_2$ such that $\beta_n$ is the
equivalent to $\alpha_n$ and the length of the formulas in the latter
sequence is increasing better than exponentially.


We give one example of succinctness results here. Consider the language
$\lan{L}_{EK}$.
Of course, $E_A$ can be defined using the modal
operators $K_i$ for $i \in A$.  But, as we now show, having the modal
operators $E_A$ in the language makes the language exponentially more
succinct.

\begin{theorem}\label{thm:expsuccinct} The language $\lan{L}_{EK}$ is
exponentially more
succinct than $\lan{L}_{K}$ on $\mathcal{X}$, for all $X$ between
$\mathcal{K}$ and $\mathcal{S}5$.
\end{theorem}

\begin{proof}
Clearly, for every formula $\alpha$ in $\lan(L)_K$, there is an
equivalent formula in $\lan{L}_{EK}$ that is no longer than $\alpha$,
namely, $\alpha$ itself.
Now consider the following two sequences of formulas:
\[ \alpha_n = \neg E^n_{\{a,b\}}  \neg p\]
and
\[\beta_1 = \neg(K_a \neg p \land K_b \neg p), \mbox{ and }  \beta_n = \neg (K_a\neg \beta_{n-1} \land K_b \neg\beta_{n-1}).\]
If we take $\size{E_A \phi} = \size{A} + \size{\phi}$, then it is easy
to see that $\size{\alpha_n} = 2n + 3$, so $\size{\alpha_n}$ is
increasing linearly in $n$.  On the other hand, since $\size{\beta_n}
> 2\size{\beta_{n-1}}$, we have $\size{\beta} \ge 2^n$.  It is also
immediate from the definition of $E_{\{a,b\}}$ that $\beta_n$ is
equivalent to $\alpha_n$ for all classes $\mathcal{X}$ between
$\mathcal{K}$ and $\mathcal{S}5$.  To complete the proof, we must show
that there is no formula shorter than $\beta_n$ in $\cal{L}_K$ that is
equivalent to $\alpha_n$.  This argument is beyond the scope of this
book; see the notes for references.
\end{proof}

\subsection{Reasoning problems}

Given the machinery developed so far, we can state some basic reasoning
problems in semantic terms. They concern {\em satisfiability} and {\em
model checking}. Most of
those problems are typically considered with a specific class of models
and a
specific language in mind. So let $\mathcal{X}$ be some class of 
models, and let $\lan{L}$ be a language.

\subsubsection{Decidability Problems}

A decidability problem checks some input for some property, and returns `yes' or `no'.
\begin{definition}[Satisfiability]\label{chap1:def:satisfiability}
The satisfiability problem for $\mathcal{X}$ is the following reasoning problem.

\fbox{
\begin{tabular}{lp{3.3in}}
{\sc Problem}: & satisfiability in $\mcx$, denoted {\sc sat}$_\mcx$.\\
{\sc Input}:  & a formula $\varphi \in \lan{L}$.\\
{\sc Question}: & does there exist a model $M \in \mathcal{X}$ and a state $s \in \Domain(M)$  such that $M,s \models \varphi$? \\
{\sc Output}: & `yes' or `no'.\\
\end{tabular}
}
\mbox{ }\\
\end{definition}

Obviously, there may well be formulas that are satisfiable
in some Kripke model (or generally, in a class $\mathcal{Y}$), but
not in $\mathcal{S}5$ models.
Satisfiability in $\mcx$ is closely related to the problem of {\em
validity} in $\mathcal{X}$, due to the following equivalence: $\varphi$
is valid in $\mcx$ iff $\neg \varphi$ is not satisfiable in $\mcx$.\\

\fbox{
\begin{tabular}{lp{3.3in}}
{\sc Problem}: & validity in $\mathcal{X}$, denoted {\sc val}$_\mcx$.\\
{\sc Input}:  & a formula $\varphi \in \lan{L}$.\\
{\sc Question}: & is it the case that $\mathcal{X}  \models \varphi$? \\
{\sc Output}: & `yes' or `no'.\\
\end{tabular}
}
\ \\

The next decision problem is computationally and conceptually simpler than the previous two, since rather than quantifying over a set of models, a specific model is given as input (together with a formula).

\begin{definition}[Model checking]\label{chap1:def:modelchecking}

The model checking problem for $\mathcal{X}$ is the following reasoning problem:\\

\fbox{
\begin{tabular}{lp{3.3in}}
{\sc Problem}: & Model checking in $\mathcal{X}$, denoted {\sc
modcheck}$_\mcx$.\\
{\sc Input}:  & a formula $\varphi \in \lan{L}$ and a pointed model
$(M,s)$ with $M \in \mathcal{X}$ and $s \in \Domain(M)$.\\
{\sc Question}: & is it the case that $M,s \models \varphi$? \\
{\sc Output}:: & `yes' or `no'.\\
\end{tabular}
}
\mbox{ }\\
\end{definition}

The field of {\em computational complexity} is concerned with the
question of how much  of a resource is needed to solve a specific
problem. The resources of most interest  are {\em
computation time} and {\em space}.
Computational complexity then asks questions of the
following form: if my input were to increase in size, how much more
space and/or time would be needed to compute the answer? Phrasing the
question this way already assumes that the problem at hand can be
solved in finite time using an algorithm, that is, that the problem is
{\em decidable}.
Fortunately, this is the case for the problems of interest to us.

\begin{proposition}[Decidability of {\sc sat} and {\sc
      modcheck}]\label{chap1:fact:decidability} 
If $\mcx$ is one of the model classes defined in
Definition~\ref{chap1:def:frameproperties}, $(M,s) \in \mcx$, and
$\varphi$ is a formula in one of the languages defined in
Definition~\ref{chap1:def:languages}, then both  {\sc
sat}$_\mcx(\varphi)$ and {\sc modcheck}$_\mcx((M,s),\varphi)$ are
decidable.
\end{proposition}

In order to say anything sensible about the additional resources that an
algorithm needs to compute the answer when the input increases in size,
we need to define a notion of size for inputs, which in our case are
formulas and models. Formulas are by definition finite
objects, but models can in principle be infinite (see, for instance,
Figure~\ref{chap1:fig:commonknowledge}). The following fact
is the key to proving Fact~\ref{chap1:fact:decidability}. For a class
of models
$\mcx$, let $\mc{F}in(\mcx) \subseteq \mcx$ be the set of models in
$\mcx$ that are finite.

\begin{proposition}[Finite model property]\label{chap1:fact:fmp}
For all classes of models in
Definition~\ref{chap1:def:frameproperties} and languages $\lan{L}$ in
Definition~\ref{chap1:def:languages}, we have, for all $\varphi \in
\lan{L}$,
\[\mcx \models \varphi \mbox{ iff } \mc{F}in(\mcx)\models \varphi.\]
\end{proposition}

Fact~\ref{chap1:fact:fmp} does not say that the models in
$\mcx$ and the finite models in $\mcx$ are the same in any meaningful
sense; rather, it says that
we do not gain valid formulas if we restrict ourselves to finite
models. It implies that a formula is satisfiable in a model in $\mcx$
iff it is satisfiable in a finite model in $\mcx$.  It follows that
in the languages we have considered so far, `having a finite
domain' is not expressible
(for if there were a formula $\phi$ that were true only of models with
finite domains, then $\phi$ would be a counterexample to
Fact~\ref{chap1:fact:fmp}).

\begin{definition}[Size of Models]\label{chap1:def:size}
For a finite model $M = \langle S,^\agents, V^\atoms\rangle$,
the size of $M$, denoted $\msize{M}$, is
the sum of the number of states ($\size{S}$, for which we also write $\size{M}$) and the number of
pairs in the accessibility relation ($\size{R_a}$) for each agent $a \in
\agents$.
\end{definition}

We can now strengthen Fact~\ref{chap1:fact:fmp} as follows.
\begin{proposition}\label{chap1:prop:smallmodel}
For all classes of models in
Definition~\ref{chap1:def:frameproperties} and languages $\lan{L}$ in
Definition~\ref{chap1:def:languages}, we have, for all $\varphi \in
\lan{L}$,
$\varphi$ is satisfiable in  $\mcx$  iff  there is a model $M \in\mcx$
such that $\size{\Domain(M)} \leq  2^{\size{\,\varphi\,}}$ and $\varphi$
is
satisfiable in $M$.
\end{proposition}
\vspace{.1cm}

The idea behind the proof of Proposition~\ref{chap1:prop:smallmodel} is
that states that `agree' on all subformulas of
$\varphi$ can be `identified'.
Since there are only $\size{\phi}$ subformulas of $\phi$, and
$2^{\size{\,\phi\,}}$ truth assignments to these formulas, the result
follows.  Of course, work needs to done to verify this intuition, and to
show that an appropriate model can be constructed in the right class
$\mcx$.

To reason about the complexity of a computation performed by an
algorithm, we distinguish various complexity classes.
If a deterministic algorithm can solve a problem in
time polynomial in the size of the input,  the problem is said to
be in {\sf P}. An example of a problem in {\sf P} is to decide,
given two finite Kripke models $M_1$ and $M_2$, whether there exists a
bisimulation between them.
Model checking for the basic multi-modal language is also in {\sf P};
see Proposition~\ref{chap1:prop:modelchecking}.

In a {\em nondeterministic} computation, an algorithm is allowed
to `guess' which of a finite number of steps to take next.
A nondeterministic
algorithm for a decision problem says `yes' or {\em accepts the input}
if the algorithm says `yes' to an appropriate sequence of guesses. So a
nondeterministic algorithm can be seen as generating different branches at each computation step, and the answer of
the nondeterministic algorithm is `yes' iff one of the branches results
in a `yes' answer.

The class {\sf NP} is the class of problems that are solvable by a
nondeterministic algorithm in polynomial time. Satisfiability of
propositional logic is an example of a problem in {\sf NP}:
an algorithm for satisfiability first guesses an appropriate truth
assignment to the primitive propositions, and then verifies that the
formula is in fact true under this truth assignment.

A problem that is at least as hard as any problem
in {\sf NP} is called {\sf NP}{\em -hard}. An {\sf NP}-hard problem has
the property that any problem in {\sf NP} can be reduced to it using a
polynomial-time reduction.
A problem is {\sf NP}{\em -complete} if it
is both in {\sf NP} and {\sf NP}-hard; satisfiability for propositional
logic is well known to be {\sf NP}-complete.
For an arbitrary complexity class {\sf C}, notions of {\sf C}-hardness
and {\sf C}-completeness can be similarly defined.

Many other complexity classes have been defined.  We mention a few of
them here. An algorithm that runs in space polynomial in the size of the
input it is in {\sf PSPACE}.  Clearly if an algorithm needs only
polynomial time then it is in polynomial space; that is {\sf P}
$\subseteq$ {\sf PSPACE}.  In fact, we also have {\sf NP} $\subseteq$
{\sf PSPACE}.  If an algorithm is in {\sf NP}, we can run it in polynomial space
by systematically trying all the possible guesses, erasing the space
used after each guess, until we eventually find one that is the `right'
guess.
{\sf EXPTIME} consists of all algorithms that run in time exponential in
the size of the input; {\sf NEXPTIME} is its nondeterministic analogue.
We have {\sf P} $\subseteq$ {\sf NP} $\subseteq$ {\sf PSPACE}
$\subseteq$ {\sf
EXPTIME} $\subseteq$ {\sf NEXPTIME}. One of the most important open
problems in computer science is the question whether {\sf P} = {\sf
NP}. The conjecture is that the two classes are different, but
this has not yet been proved; it is possible that a polynomial-time algorithm
will be found
for an {\sf NP}-hard problem.
What is
known is that {\sf P} $\neq$
{\sf EXPTIME} and {\sf NP} $\neq$ {\sf NEXPTIME}.

The complement $\bar{P}$ of a problem $P$  is the problem in which all
the `yes' and `no' answers are reversed. Given a complexity class {\sf
C}, the class co-{\sf C} is the set of problems for which the complement
is in {\sf C}. For every deterministic class {\sf C}, we have co-{\sf C}
= {\sf C}.  For nondeterministic classes, a class and its complement
are, in general, believed to be incomparable.
Consider, for example, the satisfiability problem for propositional
logic, which, as we noted above, is  {\sf NP}-complete.
Since a formula $\varphi$ is valid if and only if $\neg\varphi$
is not satisfiable,
it easily follows that the validity problem for propositional logic is
co-{\sf NP}-complete.
The class of {\sf NP}-complete and co-{\sf NP}-complete problems are
believed to be distinct.

We start our summary of complexity results for decision problems in
modal logic
with model checking.

\begin{proposition}\label{chap1:prop:modelchecking}
Model checking formulas in $\lan{L}(\atoms,\operators,\agents)$, with
$\operators = \{K_a \mid a \in \agents\}$, in finite models is in
{\sf P}.
\end{proposition}
\begin{proof}
We now describe an algorithm that, given a model $M = \langle S,
R^\agents,$ $V^\atoms\rangle$ and a formula $\varphi \in \lan{L}$,
determines
in time polynomial in $\size{\phi}$ and $\msize{M}$ whether
$M,s\models\varphi$. Given $\varphi$, order the subformulas $\varphi_1,
\dots \varphi_m$ of $\varphi$ in such a
way that, if $\varphi_i$ is a subformula of $\varphi_j$, then $i <
j$. Note that $m \leq {\size{\varphi}}$. We claim that
\begin{quotation}
\noindent (*) for every $k \leq m$, we can label each state $s$ in $M$
with either
 $\varphi_j$ (if $\varphi_j$ if true at $s$) or $\neg\varphi_j$
(otherwise), for every $j \leq k$, in $k \msize{M}$ steps.
\end{quotation}
We prove (*) by induction on $m$. If $k = 1$, $\varphi_m$ must be
a primitive proposition, and obviously we need only ${\size{M}} \leq
{\msize{M}}$ steps to label all states as required. Now suppose (*)
holds for some $k < m$,
 and consider the case $k+1$. If $\varphi_{k+1}$ is a primitive proposition, we
 reason as before. If $\varphi_{k+1}$ is a negation, then it must be
 $\neg \varphi_j$ for some $j \leq k$. Using our assumption, we know
 that the collection of formulas $\varphi_1, \dots, \varphi_k$ can be
 labeled in $M$ in $k \msize{M}$ steps. Obviously, if we include
 $\varphi_{k+1} = \neg\varphi_j$ in the collection of formulas, we can
 do the labelling in $k$ more steps: just use the opposite label for
 $\varphi_{k+1}$ as used for $\varphi_i$. So the collection $\varphi_1,
 \dots, \varphi_{k+1}$ can be labelled in $M$ in at $(k+1)\msize{M}$
steps, are required. Similarly, if
$\varphi_{k+1} = \varphi_i \land \varphi_j$, with $i, j \leq k$,
a labelling for the collection $\varphi_1, \dots, \varphi_{k+1}$ needs only
$(k+1)\msize{M}$ steps: for the last formula, in each state $s$
 of $M$, the labelling can be completed using the labellings for
$\varphi_i$ and $\varphi_j$. Finally, suppose $\varphi_{k+1}$ is of the
form $K_a\varphi_j$ with $j \leq k$. In this case, we label a state $s$
with $K_a\varphi_j$ iff each state $t$ with $R_ast$ is labelled
$\varphi_j$. Assuming the labels $\varphi_j$ and $\neg\varphi_j$ are
 already in place, this can be done in $\size{R_\agent(s)} \leq
 \msize{M}$ steps.
 \end{proof}

Proposition~\ref{chap1:prop:modelchecking} should be interpreted with
care. While having a poly\-no\-mi\-al-time procedure seems attractive, we are
talking about computation time polynomial {\em in the size of the
input}. To model an interesting scenario or system often requires `big
models'. Even for one agent and $n$ primitive propositions, a model might
consist of $2^n$ states.  Moreover, the procedure does not check 
properties of the model either, for instance whether it belongs to a
given class $\mcx$. 

\putaway{
\begin{algorithm}\caption{A model checking algorithm}\label{chap1:alg:modcheck}
\begin{program}
\mbox{The model checking algorithm } |modcheck|(M,s,\varphi)
\BEGIN %
\IF s \in |collect|(M,\varphi) \THEN |`yes'| \ELSE |`no'| \FI;
\WHERE
\PROC |collect|(M,\varphi) \BODY
          \IF \varphi = p \THEN \{s \mid V(s)(p)\} \FI;
           \IF \varphi = \varphi_1 \land \varphi_2 \THEN |collect|(M,\varphi_1) \cap |collect|(M,\varphi_2) \FI;
           \IF \varphi = \neg \psi \THEN S \setminus |collect|(M,\psi)\FI;
           \IF \varphi = K_a\psi \THEN |pre|(R_a,|collect|(M,\psi) \FI \ENDPROC
           \WHERE
           \PROC |pre|(R_a,T) = \{s \in S \mid \forall t (R_ast \Rightarrow t \in T)\} \ENDPROC
\END
\end{program}
\end{algorithm}
}

We now formulate results for satisfiability checking. The results depend
on two parameters: the class of models considered (we focus on $\mc{K},
\mc{T},
\mc{S}4, \mc{KD}45$ and $\mc{S}5$) and the language. Let $\agents_{=1}$
consist of only one agent, let $\agents_{\geq 1} \neq \emptyset$ be an
arbitrary set of agents,  and let $\agents_{\geq 2}$ be a set of at
least two agents. Finally, let $\operators = \{K_a \mid a \in
\agents\}$.

\begin{theorem}[Satisfiability]\label{chap1:thm:sat}
The complexity of the satisfiability problem is
\begin{enumerate}
\item\label{chap1:thm:sat:item:one} {\sf NP}-complete if $\mcx \in
\{\mc{KD}45, \mc{S}5\}$ and $\lan{L} =
\lan{L}(\atoms,\operators,\agents_{=1})$;
\item\label{chap1:thm:sat:item:two} {\sf PSPACE}-complete if
\begin{enumerate}
\item $\mcx \in \{\mc{K}, \mc{T}, \mc{S}4\}$ and $\lan{L} = \lan{L}(\atoms,\operators,\agents_{\geq 1})$, or
\item $\mcx \in \{\mc{KD}45, \mc{S}5\}$ and $\lan{L} =
\lan{L}(\atoms,\operators,\agents_{\geq 2})$;
\end{enumerate}
\item\label{chap1:thm:sat:item:three}
 {\sf EXPTIME}-complete if
\begin{enumerate}
\item $\mcx \in \{\mc{K}, \mc{T}$ and $\lan{L} = \lan{L}(\atoms,\operators \cup \{C\},\agents_{\geq 1})$, or
\item $\mcx \in \{\mc{S}4, \mc{KD}45, \mc{S}5\}$ and $\lan{L} =
\lan{L}(\atoms,\operators\cup \{C\},\agents_{\geq 2})$.
\end{enumerate}
\end{enumerate}
\end{theorem}

From the results in Theorem~\ref{chap1:thm:sat}, it follows that
the satisfiability problem for logics of knowledge and belief  for one
agent,
$\mc{S}5$ and $\mc{KD}45$, is exactly as hard as the satisfiability
problem for  propositional
logic. If we do not allow for common knowledge, satisfiability for the
general case is {\sf PSPACE}-complete, and with common knowledge it is
{\sf EXPTIME}-complete. (Of course, common knowledge does not add
anything for the case of one agent.)

For validity, the consequences of Theorem~\ref{chap1:thm:sat} are as
follows. We remarked earlier that if satisfiability (in $\mcx$) is in
some class {\sf C}, then validity is in co-{\sf C}. Hence,
checking validity for the cases in item~\ref{chap1:thm:sat:item:one} is
co-{\sf NP}-complete. Since co-{\sf PSPACE} = {\sf PSPACE}, the
validity problem for the cases in item~\ref{chap1:thm:sat:item:two}
is {\sf PSPACE}-complete, and, finally, since co-{\sf EXPTIME} = {\sf
EXPTIME}, the
validity problem for the cases in item~\ref{chap1:thm:sat:item:three} is
{\sf EXPTIME}-complete.
What these results on satisfiability and validity mean in practice?
Historically, problems that were not in {\sf P} were viewed as too hard
to deal with in practice. However, recently, major advances have been
made in finding algorithms that deal well with many {\sf NP}-complete
problems, although no generic approaches have been found for dealing
with problems that are co-{\sf NP}-complete, to say nothing of problems
that are {\sf PSPACE}-complete and beyond.  Nevertheless, even for
problems in these complexity classes, algorithms with humans in the loop
seem to provide useful insights.  So, while these complexity results
suggest that it is unlikely that we will
be able to find tools that do automated satisfiability or
validity checking and are guaranteed to always give correct results for
the logics  that we focus on in this book, this should not be taken to
say that we cannot write algorithms for satisfiability, validity,
or model checking that are useful for the problems of practical
interest.  Indeed, there is much work focused on just that.


\subsection{Axiomatisation}\label{chap1:subsec:axioms}
In the previous section, the formalisation of reasoning was defined
around the notion of {\em truth}: $\mcx \models \varphi$ meant that
$\varphi$ is true in all models in $\mcx$.
In this section, we discuss a form of reasoning
where a conclusion is inferred
purely based on its
syntactic form. Although there are several ways to do this,
in epistemic logic, the most popular way to define deductive
inference is by defining a {\em Hilbert-style axiom system}.
Such systems provide a very simple notion of formal proofs. Some
formulas are valid merely because they have a certain syntactic form.
These are the axioms of the system. The rules of the system say that
one can conclude that some formula is valid due to other formulas
being valid.
A formal proof or \emph{derivation} is a list of formulas,
where each formula is either an axiom of the system or can be
obtained by applying an inference rule of the system to formulas that
occur earlier in the list.
A proof or derivation of $\phi$ is a derivation whose  last formula is $\phi$.

\subsubsection{Basic system}

Our first definition of such a system will make the notion more
concrete. We give our definitions for a language
where the modal operators are  $K_\agent$ for the agents in some set
 $\agents$, although many of the ideas generalise to a setting with arbitrary
 modal  operators.

\begin{definition}[System $\prok$]
Let $\lan{L} = \lan{L}(\atoms,\operators,\agents)$, with $\operators =
\{K_a \mid a \in \agents\}$. The axiom system $\prok$
consists of the following axioms and rules of inference:
\begin{center}

\framebox{

\begin{tabular}[b]{llr}
$\axiom{1}$ & All substitution instances of propositional
tautologies.\index{axiom!propositional tautologies}&  \\
$\axiom{K}$ & $K_\agent (\phi \imp \psi) \imp (K_\agent \phi\imp
K_\agent \psi)$ for all $\agent \in \agents$. \\
$\axiom{MP}$ &
From $\phi$ and $\phi \imp \psi$ infer $\psi$.\\
$\axiom{Nec}$ & From $\phi$ infer $K_\agent \phi$. \\
\end{tabular}}
\end{center}
\end{definition}

Here, formulas in the axioms $\axiom{1}$ and $\axiom{K}$ have to be
interpreted as {\em axiom schemes}: axiom $\axiom{K}$ for instance denotes all
formulas $\{K_\agent (\phi \imp \psi) \imp (K_\agent \phi\imp K_\agent
\psi)\mid \varphi, \psi \in \lan{L}\}$.
The rule \axiom{MP} is also called {\em modus ponens};
\axiom{Nec} is called {\em necessitation}. Note that the notation for
axiom \axiom{K} and the axiom system $\prok$ are the same: the context
should make clear which is intended.

To see how an axiom system is actually used, we need to define
the notion of {\em derivation}.

\begin{definition}[Derivation]\label{chap1:def:derivation}
Given a logical language $\lan{L}$, let \mbx\ be an  axiom system with
axioms $\axiom{Ax_1},\dots,$ $\axiom{Ax_n}$ and rules $\axiom{Ru_1}, \dots
\axiom{Ru_k}$. A {\em derivation} of $\varphi$ in {\bf X} is a finite
sequence $\varphi_1, \dots, \varphi_m$ of formulas such that: (a)
$\varphi_m = \varphi$, and (b) every $\varphi_i$ in the sequence
is either an instance of an axiom or else the result of applying
a rule to formulas in the sequence prior to $\varphi_i$. For the
rules  \axiom{MP} and \axiom{Nec}, this means the following:
\begin{description}
\item[MP] $\varphi_h = \varphi_j \rightarrow \phi_i$, for some
$h, j < i$.
\end{description}
That is, both $\phi_j$ and $\phi_j \rightarrow \phi_i$ occur in th
sequence before $\phi_i$.
\begin{description}
\item[Nec] $\varphi_i = K_a \varphi_j$, for some $j < i$;
\end{description}
If there is a derivation for $\varphi$ in \mbx\ we write $\mbx \vdash
\varphi$, or $\vdash_\mbx \varphi$, or, if the system {\bf X}
is clear from the context, we just write $\vdash \varphi$. We then also
say that $\varphi$ is a {\em theorem} of \mbx, or that \mbx\ {\em
proves} $\varphi$. The sequence $\varphi_1, \dots, \varphi_m$ is then
also called a {\em proof of} $\varphi$ {\em in }\mbx.
\end{definition}


\begin{example}[Derivation in \prok]
We first show that
\begin{equation}\label{chap1:eq:kprop}
\prok \vdash K_\agent(\varphi \land \psi)   \rightarrow
(K_\agent\varphi\land K_\agent\psi).
\end{equation}
We present the proof as a sequence of numbered steps (so that the
formula $\phi_i$ in the derivation is given number $i$).  This allows us
to justify each step in the proof by describing which axioms, rules of
inference, and previous steps in the proof it follows from.

\noindent \(
\begin{array}{@{}l@{~}ll}
1. & (\varphi \land \psi) \rightarrow \varphi & \axiom{1}\\
2. & K_a((\varphi \land \psi)\rightarrow  \varphi) & \axiom{Nec}, 1\\
3. & K_a((\varphi \land \psi)\rightarrow  \varphi) \rightarrow (K_a(\varphi \land \psi) \rightarrow K_a\varphi) & \axiom{K}\\
4. & K_a(\varphi \land \psi) \rightarrow K_a\varphi & \axiom{MP}, 2, 3\\
5. & (\varphi \land \psi) \rightarrow \psi & \axiom{1}\\
6. & K_a((\varphi \land \psi)\rightarrow  \psi) & \axiom{Nec}, 5\\
7. & K_a((\varphi \land \psi)\rightarrow  \psi) \rightarrow (K_a(\varphi \land \psi) \rightarrow K_a\psi) & \axiom{K}\\
8. & K_a(\varphi \land \psi) \rightarrow K_a\psi & \axiom{MP}, 6, 7\\
9. & (K_a(\varphi \land \psi) \rightarrow K_a\varphi) \rightarrow \\
  &((K_a(\varphi \land \psi) \rightarrow K_a\psi) \rightarrow (K_a(\varphi \land \psi) \rightarrow (K_a\varphi \land K_a \psi) )) & \axiom{1}\\
10. & (K_a(\varphi \land \psi) \rightarrow K_a\psi) \rightarrow (K_a(\varphi \land \psi) \rightarrow (K_a\varphi \land K_a \psi)) &\axiom{MP}, 4, 9\\
11. & K_a(\varphi \land \psi) \rightarrow (K_a\varphi \land K_a\psi) & \axiom{MP}, 8,10 \\
\end{array}
\)

Lines 1, 5, and 9 are instances of propositional tautologies (this can
be checked using a truth table).
Note that the
tautology on line 9 is of the form $(\alpha \rightarrow \beta)
\rightarrow ((\alpha \rightarrow \gamma) \rightarrow (\alpha \rightarrow
(\beta\land\gamma)))$. A proof like that above may look cumbersome, but
it does show what can be done using only the axioms and rules of
$\prok$. It is convenient to give names to properties that are
derived, and so build a library of theorems. We have, for
instance that $\prok \vdash \axiom{KCD}$, where $\axiom{KCD}$
(`$K$-over-conjunction-distribution') is
\[
\axiom{KCD} \ \ K_a(\alpha \land \beta) \rightarrow K_a\alpha \mbox{ and
} K_a(\alpha \land \beta) \rightarrow K_a\beta.
\]
The proof of this follows steps 1 - 4 and steps 5 - 8,
respectively, of the proof above. We can also derive new rules;
for example, the following rule:
\axiom{CC} (`combine conclusions')  is derivable in $\prok$:
\[
\axiom{CC} \ \  \mbox{from } \alpha \rightarrow \beta \mbox{ and } \alpha
\rightarrow \gamma \mbox{ infer } \alpha \rightarrow (\beta \land
\gamma).
\]
The proof is immediate from the tautology on line 9 above, to which
we can, given the assumptions, apply modus ponens twice.
We can give a more compact proof of
$K_a(\varphi \land \psi) \rightarrow (K_a\varphi \land K_a\psi)$ using this library:

\[
\begin{array}{lll}
1. & K_a(\varphi \land \psi) \rightarrow K_a\varphi &
\axiom{KCD}\\
2. & K_a(\varphi \land \psi) \rightarrow K_a\psi & \axiom{KCD}\\
3. & K_a(\varphi \land \psi) \rightarrow (K_a\varphi \land K_a\psi) &
\axiom{CC}, 1, 2 
\end{array}
\]
\end{example}

For every class $\mcx$ of models introduced in the previous
section,
we want to have an inference system $\mbx$ such that derivability in
\mbx\ and validity in \mcx\ coincide:

\begin{definition}[Soundness and Completeness]\label{chap1:def:completeness}
Let $\lan{L}$ be a language, let $\mcx$ be a class of models, and let
$\mbx$ be an
axiom system. The axiom system is said to be
\begin{enumerate}
\item {\em sound} for $\mcx$ and the language $\lan{L}$ if, for all
formulas $\varphi \in \lan{L}$,
$\mbx\vdash \varphi$ implies $\mcx \models \varphi$; and
\item {\em complete} for $\mcx$ and the language $\lan{L}$ if, for all
formulas $\varphi \in \lan{L}$,
$\mcx\models \varphi$ implies $\mbx \vdash \varphi$.

\end{enumerate}
\end{definition}
We now provide axioms that characterize some of the subclasses of models
that were introduced in Definition~\ref{chap1:def:frameproperties}.

\begin{definition}[More axiom systems]\label{chap1:def:axioms}
Consider the following axioms, which apply for all agents $a \in
\agents$:
\begin{center}
\framebox{
\(
\begin{array}{llr}
\axiom{T}. & K_a\varphi \rightarrow \varphi&  \\
\axiom{D}. & M_a \top &  \\
\axiom{B}. & \varphi \rightarrow K_aM_a\varphi\\
\axiom{4}. & K_a \varphi \rightarrow K_aK_a\varphi\\
\axiom{5}. & \neg K_a \varphi \rightarrow K_a\neg K_a \varphi\\
\end{array}\)}
\end{center}
A simple way to denote axiom systems is just to add the axioms that are
included together with the name $\prok$. Thus, $\axiom{KD}$ is the
axiom system that has all the axioms and rules of the system $\prok$
($\axiom{1}$, $\axiom{K}$, and rules $\axiom{MP}$ and $\axiom{Nec}$)
together with $\axiom{D}$. Similarly, $\axiom{KD45}$ extends
$\prok$ by adding the
axioms $\axiom{D}$, $\axiom{4}$ and $\axiom{5}$. System $\axiom{S4}$ is
the more common way of denoting $\axiom{KT4}$, while $\axiom{S5}$ is the
more common way of denoting $\axiom{KT45}$. If it is
necessary to make explicit that there are $m$ agents in $\agents$,
we write $\axiom{K}_m$, $\axiom{KD}_m$, and so on.
\end{definition}


\subsubsection{Using $\mathbf{S5}$ to model knowledge}\label{subsubsec:s5}
The system $\mathbf{S5}$ is an extension of $\prok$
with the so-called `properties of knowledge'. Likewise,
$\mathbf{KD45}$ has been viewed as characterizing the `properties of belief'. 
The axiom
$\axiom{T}$ expresses that knowledge is {\em
veridical}\index{veridicality}: whatever one knows, must be true.
(It is sometimes called the {\em truth axiom}.)
The
other two axioms specify so-called {\em introspective
agents}\index{introspection}: $\axiom{4}$ says that an agent knows what
he knows (positive introspection), while $\axiom{5}$ says that he knows
what he does not know (negative introspection). As a side remark, we
mention that
axiom $\axiom{4}$ is superfluous in $\axiom{S5}$; it can be
deduced from the other axioms.


All of these axioms are idealisations, and indeed, logicians do not
claim that they hold for all possible interpretations of knowledge. It is only
human to claim one day that you know a certain fact, only to find
yourself admitting the next day that you were wrong, which undercuts
the axiom ${\mathbf T}$. Philosophers use
such examples to challenge the notion of knowledge in the
first place (see the notes at the end of the chapter for references to
the literature on logical 
properties of knowledge). Positive introspection has also been viewed
as problematic.  For example, consider a pupil who is asked a question
$\phi$ to which he 
does not know the answer. It may well be that, by asking more
questions, the pupil becomes able to answer that $\phi$ is true.
Apparently, the pupil knew
$\varphi$, but was not aware he knew, so did not know that
he knew $\varphi$. 

The most debatable among the axioms is that of
negative introspection. Quite possibly, a reader of this chapter does
not know (yet) what Moore's paradox is (see
Chapter~\chapref{chap:dynamicepistemiclogic}), but did she know before
picking up this book that she did not know that? 

Such examples suggest that a reason for ignorance can be lack of
{\em awareness}.  Awareness is the subject of
Chapter~\chapref{chap:awareness} in this book.
Chapter~\chapref{chap:onlyknowing} also has an interesting link to
negative introspection: this chapter tries to capture what it means to
claim `All I know is $\phi$'; in other words, it tries to give an
account of `minimal knowledge states'. This is a tricky concept in the
presence of axiom $\axiom{5}$, since all ignorance immediately leads to
knowledge!

One might argue that `problematic' axioms for knowledge should just be omitted,
or perhaps weakened, to obtain an appropriate system for knowledge,
but what about the basic principles of modal logic: the axiom
$\axiom{K}$ and the rule of inference $\axiom{Nec}$.  How acceptable are they
for knowledge? As one might expect, we should not take anything for
granted. $\axiom{K}$ assumes perfect reasoners, who can infer logical
consequences of their knowledge. It implies, for instance, that under
some mild assumptions, an agent will know what day of the week
July 26, 5018 will be. All that it takes to answer this question is that (1)
the agent knows today's date and what day of the week it is today,
(2) she knows the rules for assigning dates, computing leap years, and
so on (all of which can be encoded as axioms in an epistemic logic with
the appropriate set of primitive propositions).
By applying $K$ to this collection of facts, it follows that the agent
must know what day of the week it will be on July 26, 5018.
Necessitation assumes agents can infer all
$\mathbf{S5}$ theorems: agent $a$, for instance, would know that
$K_b(K_bq \land \neg K_b (p \rightarrow \neg K_bq))$ is equivalent to
$(K_bq \land M_bp)$.
Since even telling whether a formula is propositionally valid is
co-{\sf NP}-complete, this does not seem so plausible.

The idealisations mentioned in this paragraph are
often summarised as {\em logical omniscience}: our $\mathbf{S5}$ agent
would know everything that is logically deducible. Other
manifestations of logical omniscience are the equivalence of
$K(\varphi \land \psi)$ and $K\varphi \land K\psi$,
and the
derivable rule in $\mathbf{K}$ that allows one to infer $K\varphi
\rightarrow K\psi$ from $\varphi \rightarrow \psi$ (this says that
agents knows all logical consequences of their knowledge).

The fact that, in reality, agents are {\em not} ideal
reasoners, and not logically omniscient, is sometimes a feature
exploited by computational systems. Cryptography for instance is
useful because artificial or human intruders are, due to their limited
capacities, not able to compute the prime factors of a large number
in a
reasonable amount of time. Knowledge, security, and cryptographic
protocols are discussed in Chapter~\chapref{chap:knowledgeandsecurity}

Despite these problems, the \axiom{S5} properties are a useful idealisation of
knowledge
for many applications in distributed computing and
economics, and have been shown to give insight into a number of
problems.
The ${\mathbf S5}$ properties are reasonable for many of the examples that we have
already given; here is one more.  Suppose that we have two
processors, $a$ and $b$, and that they are involved in computations of
three variables, $x, y$, and $z$. For simplicity, assume that the variables
are Boolean, so that they are either 0 or 1. Processor $a$ can read the
value of $x$ and of $y$, and
$b$ can read $y$ and $z$. To model this, we use, for instance, $010$ as
the state where $x = 0 = z$, and $y = 1$. Given our assumptions regarding what
agents can see, we then have $x_1y_1z_1 \sim_a x_2y_2z_2$ iff $x_1 =
x_2$ and $y_1 = y_2$. This is a simple manifestation of an {\em
interpreted system}, where the accessibility relation is based on what an
agent can see in a state. Such a relation is an equivalence relation.
Thus, an interpreted system satisfies all the knowledge
axioms.  (This is formalised in Theorem~\ref{chap1:thm:completeness}(1) below.)

While $\axiom{T}$ has traditionally been considered an appropriate axiom
for
knowledge, it has not been considered appropriate for \emph{belief}.  To
reason about belief, $\axiom{T}$ is typically replaced by the weaker
axiom
$\axiom{D}$\index{axiom!$D$}: $\neg B_\agent\bot$, which says that
the agent does not believe a contradiction; that is, the agent's
beliefs are consistent.  This gives us the axiom system $\axiom{KD45}$.
We can replace $\axiom{D}$ by the following axiom $\axiom{D'}$ to get an
equivalent axiomatisation of belief:
$$\axiom{D}': \ \ K_\agenta\varphi \rightarrow \neg K_\agenta
\neg\varphi.$$
This axioms says that the agent cannot know (or believe) both a fact and
its negation.
Logical systems that have operators for both knowledge and belief often include the axiom $K_\agent \varphi \rightarrow B_\agent\varphi$, saying that knowledge entails belief.
\subsubsection{Axiom systems for group knowledge}
If we are interested in formalising the knowledge
of just one agent $a$, the language $\lan{L}(\atoms,\{K_a\},\agents)$
is arguably too rich.
In the logic
$\axiom{S5}_1$ it can be shown that every formula is equivalent to a
{\em depth-one} formula, which has no nested occurrences of $K_a$.
This follows from the following equivalences, all of which are valid in
$\mathcal{S}5$ as well as being theorems of $\axiom{S5}$:
$KK\varphi \leftrightarrow K\varphi$; $K\neg K\varphi \leftrightarrow
\neg K\varphi$; $K(K\varphi \lor \psi)  \leftrightarrow (K\varphi \lor
K\psi)$; and $K(\neg
K\varphi \lor \psi) \leftrightarrow \neg K\varphi \lor K\psi$.  From a logical perspective things become more interesting in the
multi-agent setting.

We now consider axiom systems for the notions of group knowledge that
were defined earlier.  Not surprisingly, we need some additional axioms.

\begin{definition}[Logic of common knowledge]
The following axiom and rule capture common knowledge.

\begin{center}
\framebox{
\(
\begin{array}{ll}
\axiom{Fix}. &C_\Group \phi\rightarrow E_\Group(\phi \land
C_\Group\phi).\\
\axiom{Ind}. &\mbox{From $\phi \rightarrow E_\Group(\psi \land \phi)$ infer
$\phi \rightarrow C_\Group\psi$}.
\end{array}
\)
}
\end{center}
For each axiom system {\bf X} considered earlier, let $\axiom{XC}$ be
the result of adding $\axiom{Fix}$ and $\axiom{Ind}$ to $\axiom{X}$.
\end{definition}


The {\em fixed point axiom} \axiom{Fix} says that common knowledge can
be viewed as the fixed point of an equation: common knowledge of $\phi$
holds if everyone knows both that $\phi$ holds and that $\phi$ is common
knowledge.  \axiom{Ind} is called the {\em induction rule}; it can be
used to derive common knowledge `inductively'.  If it is the case that
$\phi$ is `self-evident', in the sense that if it is true, then
everyone knows it, and, in addition, if $\phi$ is true, then everyone
knows $\psi$, we can show by induction that if $\phi$ is true, then so
is $E_\Group^k(\psi \land \phi)$ for all $k$.  It follows that $C_\Group
\phi$ is true as well.  Although common knowledge was defined as an
`infinitary' operator, somewhat surprisingly, these axioms completely
characterize it.

\commentout{
For the logic of common belief, we need to adjust the common knowledge axioms so that common belief does not imply veridicality of belief. In fact, there are alternative formulations of these axiomatisations that equally apply to common knowledge and common belief. We only list the changed axioms and rules. For general belief we write $EB_\Group$.

\bigskip

\begin{center}
\framebox{
\(
\begin{array}{ll}
\axiom{Emix} & CB_\Group \phi \imp EB_\Group (\phi \et CB_\Group \phi)\\
\axiom{Eind} & CB_\Group (\phi \imp EB_{\Group} \phi) \imp (EB_\Group \phi \imp CB_\Group \phi)
\end{array}
\)
}
\end{center}
\bigskip
}

For distributed knowledge, we consider the following axioms for all $A
\subseteq \agents$:

\begin{center}
\framebox{
\(
\begin{array}{ll}
\axiom{W}. & K_\agent \phi \rightarrow D_\Group \phi \ \mbox{if $a
\in  \Group$}.\\
\axiom{K_D}. & D_\Group(\phi \rightarrow \psi) \rightarrow (D_\Group\phi
\rightarrow D_\Group \psi).\\
\axiom{T_D}. & D_\Group\varphi \rightarrow \varphi. \\
\axiom{D_D.} & \neg D_\Group \neg \top.\\
\axiom{B_D}. & \varphi \rightarrow D_\Group \neg D_\Group \neg \varphi.\\
\axiom{4_D}. & D_\Group \varphi \rightarrow D_\Group D_\Group\varphi.\\
\axiom{5_D}. & \neg D_\Group\varphi \rightarrow D_\Group\neg
D_\Group\varphi.\\

 \end{array}
\)
}
\end{center}
These axioms have to be understood as follows. It may help to think
about distributed knowledge in a group $\Group$ as the knowledge of a
wise man, who has been told, by every member of $\Group$, what each of
them knows. This is captured by axiom $\axiom{W}$. The other axioms
indicate that the wise man has at least the same reasoning abilities as
distributed knowledge to the system $\axiom{S5}_m$, we
add the axioms $\axiom{W}, \axiom{K_D},  \axiom{T_D}, \axiom{4_D}$, and
$\axiom{5_D}$ to the axiom system. For $\axiom{K}_m$, we add only
$\axiom{W}$
and $\axiom{K_D}$.

\subsubsection{Proving Completeness}
We want to prove that the axiom systems that we have defined are sound
and complete for the corresponding semantics; that is, that $\axiom{K}$
is sound and complete with respect to $\mathcal{K}$, $\axiom{S5}$ is
sound and complete with respect to $\mathcal{S}5$, and so on.
Proving soundness is straightforward: we prove by induction on $k$ that
any formula proved using a derivation of length $k$ is valid.
Proving completeness is somewhat harder. There are different approaches,
but the common one involves to show that if a formula is not derivable,
then there is a model in which it is false.  There is a special model
called the \emph{canonical model} that simultaneously shows this for all
formulas.  We now sketch the construction of the canonical model.

The states in the canonical model correspond to \emph{maximal consistent
sets of formulas}, a notion that we define next.  These sets provide the
bridge between the syntactic and semantic approach to validity.

\begin{definition}[Maximal consistent set]
A formula $\phi$ is {\em consistent with axiom system $\axiom{X}$} if we
cannot derive $\neg \phi$ in $\axiom{X}$.   A finite set
$\{\phi_1,\ldots, \phi_n\}$ of formulas is consistent with $\axiom{X}$
if the conjunction $\phi_1\land \ldots \land \phi_n$ is consistent with
$\axiom{X}$.  An infinite set $\Gamma$ of formulas is consistent with
$\axiom{X}$ if each
finite subset of $\Gamma$ is consistent with $\axiom{X}$.
Given a language $\lan{L}$ and an axiom system $\axiom{X}$, a
\emph{maximal consistent set} for $\axiom{X}$
and $\lan{L}$ is a set $\Gamma$ of formulas in $\lan{L}$ that is
consistent  and \emph{maximal}, in the sense that every strict
superset $\Gamma'$ of $\Gamma$ is inconsistent.
\end{definition}

We can show that a maximal consistent set $\Gamma$ has the property
that, for every formula $\phi \in \lan{L}$, exactly one of $\phi$ and
$\neg \phi$ is in $\Gamma$.
If both were in $\Gamma$, then $\Gamma$ would be inconsistent; if
neither were in $\Gamma$, then $\Gamma$ would not be maximal. A maximal
consistent set is much like a state in a Kripke model, in that every
formula is either true or false (but not both) at a state.  In fact, as
we suggested above, the states in the canonical model can be identified
with maximal consistent sets.

\begin{definition}[Canonical model]
The \emph{canonical model for $\lan{L}$ and $\axiom{X}$} is the Kripke
model
$M=\langle S, R,V \rangle$ defined as follows:
\begin{itemize}
\item $S$ is the set of all maximal consistent sets for $\axiom{X}$ and
$\lan{L}$;
\item $\Gamma R_a \Delta$ iff $\Gamma \vert K_a \subseteq \Delta$, where
$\Gamma \vert K_a = \{ \varphi \mid K_a \phi \in \Gamma \}$;
\item $V(\Gamma)(p)= \mathit{true}$ iff $p \in \Gamma$.
\end{itemize}
\end{definition}

The intuition for the definition of $R_a$ and $V$ is easy to explain.
Our goal is to show that the canonical model satisfies what is called
the \emph{Truth Lemma}: a formula $\phi$ is true at a state $\Gamma$ in
the canonical model iff $\phi \in \Gamma$.
(Here we use the fact
that the states in the canonical model are actually sets of
formulas---indeed, maximal consistent sets.)  We would hope to prove
this by induction.  The definition of $V$ ensures that the Truth Lemma
holds for primitive propositions.  The definition of $R_a$ provides a
necessary condition for the Truth Lemma to hold for formulas of the form
$K_a \phi$.  If $K_a \phi$ holds at a state (maximal consistent set)
$\Gamma$ in the canonical model, then $\phi$ must hold at all states
$\Delta$ that are accessible from $\Gamma$.  This will be the case if
$\Gamma \vert K_a \subseteq \Delta$ for all states $\Delta$ that are accessible
from $\Gamma$ (and the Truth Lemma applies to $\phi$ and $\Delta$).

The Truth Lemma can be shown to hold for the canonical model,
as long as we consider a language that does not involve common knowledge
or distributed knowledge.   (The hard part comes in showing that if
$\neg K_a \phi$ holds at a state $\Gamma$, then there is an accessible
state $\Delta$ such that $\neg \phi \in \Delta$.  That is, we must show
that the $R_a$ relation has `enough' pairs.)
In addition to the Truth Lemma, we can also show that the canonical
model for axiom system $\axiom{X}$ is a model in the corresponding class
of models; for
example, the canonical model for $\axiom{S5}$ is in $\mathcal{S}5$.

Completeness follows relatively easily once these two facts are
established.  If a formula $\phi \in \lan{L}$ cannot be derived
in $\axiom{X}$ then $\neg \phi$ must be consistent with $\axiom{X}$, and
thus can be shown to be an element of a maximal consistent set, say
$\Gamma$.  $\Gamma$ is a state in the canonical model for $\axiom{X}$
and $\lan{L}$.  By the Truth Lemma, $\neg \phi$ is true at $\Gamma$, so
there is a model where $\phi$ is false, proving the completeness of
$\axiom{X}$.

This argument fails if the language includes the common knowledge
operator.  The problem  is that with the common knowledge operator in
the language, the logic is
not \emph{compact}: there is a set of formulas such that all its finite
subsets are satisfiable, yet the whole set is not satisfiable.
Consider the set $\{E_\Group^n p \mid n \in \mathbb{N}\} \cup \{\neg C_\Group
p\}$, where $\Group\subseteq\agents$ is a group with at least two agents.  Each finite subset
of this set is easily seen to be satisfiable
in a model in $\mathcal{S}5$ (and hence in a model in any of the other
classes we have considered), but the whole set of formulas is not
satisfiable in any Kripke model.
Similarly, each finite subset of this set can be shown to be consistent
with $\axiom{S5C}$.   Hence, by definition, the whole set is
consistent with $\axiom{S5C}$ (and hence all other axiom systems we have considered).
This means that this set must be a subset of a maximal
consistent set.  But, as we have observed, there is no
Kripke model where this set of formulas is satisfied.


This means that a different proof technique is necessary to prove
completeness.
Rather than constructing
one large canonical model for all formulas, for each formula $\phi$, we
construct a finite canonical model tailored to $\phi$.  And rather than
considering maximal consistent subsets  to the set of all formulas in the
language, we consider maximal consistent sets of the set of
subformulas of $\phi$.

The canonical model $M_\phi = \langle S_\phi, R, V \rangle$ for $\phi$
and $\axiom{KC}$ is defined as follows:
\begin{itemize}
\item $S_\phi$ is the set of all maximal consistent sets of subformulas
of $\phi$ for $\axiom{KC}$;
\item $\Gamma R_a \Delta$ iff $(\Gamma \vert K_a)
 \cup \{C_\Group  \psi \mid C_\Group \psi \in \Gamma $ and $a \in \Group\} \subseteq \Delta$.
\item $V(\Gamma)(p)= \mathit{true}$ iff $p \in \Gamma$.
\end{itemize}
The intuition for the modification to the definition of $R_a$ is the
following:  Again, for the Truth Lemma to hold, we must have
$\Gamma \vert K_a  \subseteq \Delta$, since if
$K_a\psi \in \Gamma$, we want $\psi$ to hold in all states accessible
from $\Gamma$.  By the fixed point axiom, if $C_\Group \psi$ is true
at a state $s$, so is $E_\Group C_\Group \psi$; moreover, if $a \in
\Group$, then $K_a C_\Group \psi$ is also true at $s$.  Thus, if
$C_\Group \psi$ is true at $\Gamma$, $C_\Group \psi$ must also be true at all
states accessible from $\Gamma$, so we must have $\{C_\Group \psi \mid C_\Group
\psi \in \Gamma $ and $ a \in \Group \} \subseteq \Delta$.  Again,
we can show that the Truth Lemma holds for the canonical model for
$\phi$ and $\axiom{KC}$ for subformulas of $\phi$; that is, if $\psi$ is
a subformula of
$\phi$, then $\psi$ is true at a state $\Gamma$ in the canonical model
for $\phi$ and $\axiom{KC}$ iff $\phi \in \Gamma$.

We must modify this construction somewhat for axiom systems that contain
the axiom $\axiom{4}$ and/or $\axiom{5}$.  For axiom systems that contain
$\axiom{4}$, we redefine $R_a$ so that $\Gamma R_a \Delta$ iff $(\Gamma
\mid K_a) \cup \{C_\Group \psi \mid C_\Group \psi \in \Gamma $ and $ a \in \Group\} \cup
\{ K_a \psi \mid K_a \psi \in \Gamma \}
\subseteq \Delta$.
The reason that we want $\{K_a \varphi \mid K_a \phi \in \Gamma \}
\subseteq \Delta$ is that if $K_a \psi$ is true at the state
$\Gamma$, so is $K_a K_a \psi$, so $K_a \psi$ must be true at all
worlds accessible from $\Gamma$.  An obvious question to ask is why we
did not make this requirement in our original canonical model
construction.  If both $K_a \psi$ and $K_a K_a \psi$ are subformulas of
$\phi$, then the requirement is in fact not necessary.  For if $K_a \psi
\in \Gamma$, then consistency will guarantee that $K_a K_a \psi$ is as
well, so the requirement that $\Gamma \vert K_a \subseteq \Delta$
guarantees that $K_a \psi \in \Delta$.  However, if $K_a \psi$ is a
subformula of $\phi$ but $K_a K_a \psi $ is not, this argument fails.

For systems that contain $\axiom{5}$, there are further subtleties.  We
illustrate this for the case of $\axiom{S5}$. In this case, we require that
$\Gamma R_a \Delta$ iff $\{K_a \psi \mid K_a \psi \in \Gamma \} =
\{K_a \psi \mid K_a \psi \in \Delta \}$ and
$\{C_\group \psi \mid C_\Group \psi \in \Gamma $ and $a \in \Group \} =
\{C_\Group \psi \mid C_\Group \psi \in \Delta $ and $a \in \Group \}$.  Notice that the fact that
$\{K_a \psi \mid K_a \psi \in \Gamma \} = \{K_a \psi \mid K_a \psi
\in \Delta \}$ implies that $\Gamma \vert K_a = \Delta \vert  K_a$.
We have already argued that having $\axiom{4}$ in the system means that
we should have  $\{K_a \psi \mid K_a \psi \in
\Gamma \} \subseteq \{K_a \psi \mid K_a \psi \in \Delta \}$.  For the
opposite inclusion, note that if $ K_a  \psi \notin \Gamma$, then
$\neg K_a \psi$ should be true at the state $\Gamma$ in the canonical
model, so (by $\axiom{5}$) $K_a \neg K_a \psi$ is true at $\Gamma$, and
$\neg K_a \psi$ is true at $\Delta$ if $\Gamma R_a \Delta$.  But this
means that $K_a \psi \notin \Delta$ (assuming that the Truth Lemma
applies).  Similar considerations show that we must have
$\{C_\Group \psi \mid C_\Group \psi \in \Gamma $ and $ a \in \Group \} = \{C_\Group \psi \mid C_\Group \psi \in
\Delta $ and $ a \in \Group \}$, using the fact that $\neg C_\Group \psi \imp E_\Group \neg C_\Group \psi$
is provable in $\axiom{S5C}$.

Getting a complete axiomatisation for languages involving distributed
knowledge requires yet more work; we omit details here.

We summarise the main results regarding completeness of epistemic logics
in the following theorem.
Recall that, for an axiom system $\axiom{X}$, the axiom system
$\axiom{XC}$ is the result of adding the axioms $\axiom{Fix}$ and
$\axiom{Ind}$ to $\axiom{X}$.
Similarly, $\axiom{XD}$ is the result of adding the `appropriate'
distributed knowledge axioms to $\axiom{X}$; specifically,
it includes the axiom $\axiom{W}$,
together with every axiom $\axiom{Y_D}$ for which $\axiom{Y}$ is
an axiom of $\axiom{X}$. So, for example, $\axiom{S5D}$ has
the axioms of $\axiom{S5}$ together with $\axiom{W}$, $\axiom{K_D}$,
$\axiom{T_D}$, $\axiom{4_D}$, and $\axiom{5_D}$.

\begin{theorem}\label{chap1:thm:completeness}
If $\lan(\atoms,\operators,\agents)$,
$\axiom{X}$ is an axiom systems that includes all the axioms and
rules of $\axiom{K}$ and some (possibly empty) subset of $\{\axiom{T},
\axiom{4}, \axiom{5}, \axiom{D}\}$, and $\mathcal{X}$ is the
corresponding class of Kripke models, then the following hold:
\begin{enumerate}
\item if $\operators = \{K_a \mid a \in \agents\}$, then
$\axiom{X}$ is sound and complete for $\mcx$ and $\lan{L}$;
\item if $\operators = \{K_a \mid a \in \agents\} \cup \{C_\Group \mid \Group
\subseteq \agents\}$, then
$\axiom{XC}$ is sound and complete for $\mcx$ and $\lan{L}$;
\item if $\operators = \{K_a \mid a \in \agents\} \cup \{D_\Group \mid \Group
\subseteq \agents\}$, then
$\axiom{XD}$ is sound and complete for $\mcx$ and $\lan{L}$;
\item if $\operators = \{K_a \mid a \in \agents\} \cup
 \{C_\Group \mid \Group \subseteq \agents\} \cup
\{D_\Group \mid \Group \subseteq \agents\}$, then
$\axiom{XCD}$ is sound and complete for $\mcx$ and $\lan{L}$.
\end{enumerate}
\end{theorem}

\section{Overview of the Book}\label{chap1:sec:overview}



The book is divided into three parts: informational
attitudes, dynamics, and applications. Part I, informational
attitudes, considers ways that basic epistemic logic can be extended
with other modalities related to knowledge and belief, such as ``only
knowing'', ``awareness'', and probability.  There are three chapters
in Part I:

\paragraph{Only Knowing}
Chapter \chapref{chap:onlyknowing}, on only knowing, is authored by Gerhard
Lakemeyer and Hector J.\ Levesque. What do we mean by `only knowing'?  
When we say that an agent knows $p$, we usually mean that the agent
knows {\em at least} $p$, but 
possibly more. In particular, knowing $p$ does not allow us to conclude that $q$ is not known.
Contrast this with the situation of a knowledge-based agent, whose knowledge base consists of $p$,
and nothing else. Here we would very much like to conclude that this agent does not know $q,$ but to
do so requires us to assume that $p$ is all that the agent knows or, as one can say, the agent only
knows $p$. In this chapter, the logic of only knowing for both single
and multiple
agents is considered, from both the semantic and proof-theoretic
perspective. It is shown that only knowing can be used
to capture a certain form of honesty, and that it relates to a form of
non-monotonic reasoning. 

\paragraph{Awareness}
Chapter \chapref{chap:awareness}, on logics where knowledge and awareness
interact, is authored by Burkhard Schipper. 
Roughly speaking, an agent is unaware of a formula $\phi$ if $\phi$ is
not on his radar screen (as opposed to just having no information
about $\phi$, or being uncertain as to the truth of $\phi$).  
The chapter discusses various  approaches to
modelling (un)awareness.
While the focus is on
axiomatisations of structures capable of modelling knowledge and
awareness, structures for modelling
probabilistic beliefs and awareness, are also discussed, as well as
structures for awareness of unawareness. 

\paragraph{Epistemic Probabilistic Logic}
Chapter \chapref{chap:probabilisticupdates}, authored by
Lorenz Demey and Joshua Sack, provides an overview of
systems that combine probability theory, which describes quantitative
uncertainty, with epistemic logic, which describes qualitative
uncertainty. 
By combining knowledge and probability, 
one obtains a very powerful account of information and
information flow. 
Three types of systems are investigated:
systems that describe uncertainty of agents at a
single moment in time, systems where the uncertainty changes
over time, and systems that describe the actions that cause
these changes. 

\bigskip

\noindent Part II on dynamics of informational attitudes considers aspects of how knowledge and
belief change over time.  It consists of three chapters:

\paragraph{Knowledge and Time}
Chapter \chapref{chap:knowledgeandtime}, on knowledge and time, is
authored by Clare Dixon, 
Cl\'audia Nalon, and Ram Ramanujam. It
discusses the dynamic aspects of knowledge, which can be 
characterized by a combination of temporal and epistemic
logics. 
The chapter presents the language and axiomatisation for such a combination,
and discusses complexity and expressivity issues.
It presents two different proof methods (which apply quite broadly):
{\em resolution}
and {\em tableaux}. Levels of knowledge and the relation between
knowledge and communication in distributed protocols are also
discussed, and  an automata-theoretic characterisation of the knowledge
of finite-state agents is provided. The chapter concludes with a brief
survey on applications. 

\paragraph{Dynamic Epistemic Logic}
Chapter \chapref{chap:dynamicepistemiclogic}, on dynamic epistemic logic,
is authored by Lawrence Moss. 
Dynamic Epistemic Logic ($\DEL$) extends epistemic logic with operators
corresponding to  \emph{epistemic actions}.   The most
basic epistemic action is
a public announcement of a given sentence to all agents. 
In the first part of the chapter, a logic called $\PAL$ 
(\emph{public announcement logic}), which includes announcement
operators, is introduced.  
Four different axiomatisations for $\PAL$ are given and compared.
 It turns out that $\PAL$ without
common knowledge is reducible to standard epistemic logic: the announcement
operators may be
translated away.  However, this changes once we include common knowledge
operators in the language. The second part of Chapter
\chapref{chap:dynamicepistemiclogic} 
is devoted to more general epistemic 
actions, such as private
announcements.   

\paragraph{Dynamic Logics of Belief Change}
Chapter \chapref{chap:beliefrevisioninDEL}, on belief change, is authored
by Johan van 
Benthem and Sonja Smets. 
The chapter gives an overview of current dynamic logics that describe
belief update and revision.  This involves a combination of ideas from
belief revision theory and dynamic epistemic logic.  The chapter
describes various types of belief change, depending on whether the
information received is `hard' or `soft'.
The chapter continues with three topics that naturally complement the
setting of single steps of belief change: connections with
probabilistic 
approaches to belief change, long-term temporal process structure including links
with formal learning theory, and multi-agent scenarios of information flow and belief
revision in games and social networks. It ends with a discussion of
alternative approaches, further directions, and windows to the broader
literature.

\bigskip
\noindent Part III considers applications of epistemic logic in
various areas.  It consists of five chapters:

\paragraph{Model Checking Temporal Epistemic Logic}
Chapter \chapref{chap:modelchecking}, authored by
Alessio Lomuscio and Wojciech Penczek, surveys 
work on model checking systems against 
temporal-epistemic specifications.  
The focus is on two approaches to verification:
approaches based on \emph{ordered binary decision diagrams} (OBDDs) and
approaches based on translating specifications to propositional logic,
and then applying propositional satisfiability checkers (these are called
\emph{SAT-based} approaches).
OBDDs provide a compact representation for propositional formulas; they
provide powerful techniques for efficient mode checking;
SAT-based model checking is the basis for many
recent symbolic approach to verification.  
The chapter also discusses some more advanced techniques for model
checking.  


\paragraph{Epistemic Foundations of Game Theory}
Chapter \chapref{chap:epistemicfoundationsforgames}, authored by Giacomo
Bonanno, provides an overview of the epistemic approach to game theory.
Traditionally, game theory focuses on interaction among intelligent,
sophisticated and rational individuals. 
The epistemic approach attempts to characterize, using epistemic notions,
the behavior of rational and intelligent players who know the
structure of the game and the preferences of their opponents and who
recognize each other's rationality and reasoning abilities. The
focus of the analysis is on the implications of common belief of rationality in
strategic-form games and on dynamic games with perfect information. 

\paragraph{BDI Logics}
Chapter \chapref{chap:agentsbdi}, on logics of beliefs, desires, and
intentions (BDI), is authored by
John-Jules Ch.\ Meyer, Jan Broersen and Andreas Herzig. 
Various formalisations of BDI in logic are considered, such as the
approach of Cohen and Levesque (recast in dynamic logic),
Rao and Georgeff's influential $\mathsf{BDI}$ logic based on the
branching-time temporal logic $\mathsf{CTL}^*$, the $\mathsf{KARO}$
framework, in which action together with knowledge (or belief) is
the primary concept on which other agent notions are built, and 
$\mathsf{BDI}$ logics based on $\mathsf{STIT}$ (seeing to it that)
logics, such as $\mathsf{XSTIT}$. 

\paragraph{Knowledge and Ability}
Chapter \chapref{chap:strategicability}, authored by
Thomas {\AA}gotnes, Valentin Goranko, Wojciech Jamroga and Michael
Wooldridge, relates epistemic logics to various
logics for {\em strategic abilities}.  It starts by discussing
approaches from
philosophy and artificial intelligence to modelling the interaction of
agents knowledge and abilities, and then focuses on concurrent game
models and the alternating-time temporal logic $ATL$. The authors 
discuss how $ATL$ enables reasoning about agents' coalitional
abilities to achieve qualitative objectives in concurrent game models,
first assuming complete information and then under incomplete
information and uncertainty about the structure of the game model. 
Finally,  extensions of $ATL$ that allow explicit reasoning
about the interaction of knowledge and strategic abilities are
considered; this leads to the notion of \emph{constructive knowledge}.

\paragraph{Knowledge and Security}
Chapter \chapref{chap:knowledgeandsecurity}, on knowledge and security, is
authored by 
Riccardo Pucella. A persistent intuition in the field of computer security says that
epistemic logic, and more generally epistemic concepts, are relevant
to the formalisation of security properties. What grounds this
intuition is that much work in the field is based on epistemic
concepts. Confidentiality, 
integrity, authentication, anonymity, non-repudiation, all can be
expressed as epistemic properties. This survey illustrates the use of
epistemic concepts and epistemic logic to formalise a specific
security property, {\em confidentiality}. Confidentiality is a prime
example of the use of knowledge to make a security property
precise. It is explored in two large domains of 
application: cryptographic protocol analysis and multi-level security
systems.

\section{Notes}\label{chap1:sec:literature}

The seminal work of the philosopher Jaakko Hintikka
(\citeyear{hintikka:62a}) is typically taken 
as the starting point of modern epistemic logic. Two texts on
epistemic logic 
by computer scientists
were published in 1995: one by \cite{fagin:95a}
and the other by \cite{meyer:95a}.
Another influential text on epistemic logic, 
which focuses more on philosophical aspects,
is by \cite{Rescher:survey}.
Formal treatments of the notion of knowledge in
artificial intelligence, in particular for reasoning about action, go
back to the work of \cite{moore:77a}.  In the
mid-1980s, the conference on Theoretical Aspects of Reasoning
About Knowledge (TARK), later renamed to ``Theoretical Aspects of
\emph{Rationality} and Knowledge, was started (\citeyear{tark});
in the mid-1990s, the Conference on Logic and Foundations of Game and
Decision Theory (LOFT) (\citeyear{loft}) began.  
These two conferences continue
to this day, bringing together computer scientists, economists, and
philosophers.

Our chapter is far from the first introduction to epistemic
logic. The textbooks by \cite{fagin:95a} and by \cite{meyer:95a} each
come with an 
introductory chapter; more recent surveys and introductions can be
found in the book by \citet[Chapter 2]{del},
in a paper on epistemic logic and epistemology  by
\cite{holliday:13a}, 
in the chapter by \cite{estructures}, which provides a survey of
semantics for epistemic notions, 
and in \defcitealias{wikipedia}{Wikipedia}%
online resources (\citealt{stanfordenc}, \citetalias{wikipedia}).

\cite{Hal1} provides an introduction to applications of knowledge in
distributed computing; the early chapters of the book by \cite{Perea12} give
an introduction to the use of epistemic logic in game theory.
As we already said, more discussion of the examples in
Section~\ref{chap1:sec:intro} can be found in the relevant chapters.
Public announcements are considered in
Chapter~\chapref{chap:dynamicepistemiclogic};
protocols are studied in Chapter~\chapref{chap:knowledgeandsecurity}
and, to some extent, in
Chapter~\chapref{chap:knowledgeandtime}; strategic ability is
the main topic of Chapter~\chapref{chap:strategicability};
epistemic foundations of game theory are considered in
Chapter~\chapref{chap:epistemicfoundationsforgames}; distributed computing
is touched on in Chapter~\chapref{chap:knowledgeandtime}, while
examples of model checking distributed protocols are given in
Chapter~\chapref{chap:modelchecking}.

The use of Kripke models puts our approach to epistemic logic firmly in
the tradition of modal logic, of which Kripke is one
of the founders (see \cite{kripke:63a}).
Modal logic has become {\em the} framework to reason
not only about notions as knowledge and belief, but also about agent
attitudes such as
desires and intentions \citep{rao:91c}, and about notions like time
\citep{emerson:90a}, action \citep{harel:84a}, programs
\citep{fischerladner}, reasoning about obligation
and permission \citep{vonWright1951-VONIDL}, and
combinations of them.  Modern references to modal logic include the
textbook by \cite{blackburn:2001a} and the handbook edited by \cite{hanbookmodal}.

Using modal logic to
formalise knowledge and belief suggests that one has an idealised
version of these notions in mind. The discussion in
Section~\ref{subsubsec:s5} is only the tip of the iceberg.  Further discussion
of logical omniscience can be found in 
\citep{Stalnaker1991-STATPO-8,INT:INT3} and in \citep[Chapter 9]{fagin:95a}.
There is a wealth of discussion in the philosophy
and psychology literature of the axioms and their reasonableness
\citep{koriat1993,larsson:2004,zangwill2013}.
Perhaps the most controversial axiom of knowledge is $\axiom{5}$;
which was dismissed in the famous claim by Donald
Rumsfeld 
that there are `unknown unknowns'
(see
  \url{http://en.wikipedia.org/wiki/There_are_known_knowns}). 
Some approaches for dealing with lack of knowledge using awareness avoid
this axiom (and, indeed, all the others); see
Chapter~\chapref{chap:awareness}.

Broadly speaking,
philosophers usually distinguish between the {\em truth} of a claim,
our {\em belief} in it, and the {\em justification for the claim}.
These are often considered the three key elements of knowledge. Indeed,
there are papers that define knowledge as justified true belief.
There has been much debate of this definition, going back to
Gettier's (\citeyear{gettier1963}) {\em Is justified true belief knowledge?}.
Halpern, Samet, and Segev (\citeyear{HSamet2}) provide a recent  
perspective on these issues.
The notion of {\em common knowledge} is often traced back  to the
philosopher David Lewis's (\citeyear{lewis:69a})
independently developed by the sociologist Morris
Friedell (\citeyear{Friedell67}).  Work on common knowledge in economics
was initiated by Robert Aumann (\citeyear{aumann:76a});
John McCarthy's (\citeyear{mccarthy:1978}) work involving common
knowledge had a significant impact in the field of artificial intelligence.
Good starting points
for further reading on the topic of common knowledge are by
\cite{stanfordenc:commonknowledge}
and by \citet[Chapter 6]{fagin:95a}.
Section~9.5
compares the notions
of common knowledge with that of common belief.

Distributed knowledge was discussed first, in an informal way, by
\cite{Hay45},
and then, in a more formal way, by \cite{Hil77}.
It was rediscovered and popularized by
\cite{halpern:84a}, who originally called it
\emph{implicit knowledge}.

The notion of bisimulation is a central notion in modal logic,
providing an answer to the question when two models are `the same'
and is discussed in standard modal logic texts
\citep{blackburn:2001a,hanbookmodal}.  Bisimulation arises quite often
in 
this book, including in Chapters~\chapref{chap:knowledgeandtime},
~\chapref{chap:dynamicepistemiclogic}, and
~\chapref{chap:beliefrevisioninDEL}.

We mentioned below Theorem~\ref{chap1:thm:sat}, when discussing complexity 
of validity, that some recent advances make NP-complete problems seem more tractable: for this
we refer to work by \citet{GKSS08}.

We end this brief discussion of the background literature by providing
the pointers to the technical results mentioned in our
chapter. Theorem~\ref{chap1:thm:validities:item:one} gives some standard
valid formulas for several classes of models (see \citet[Chapter
2.4]{fagin:95a} for a textbook treatment).
Theorem~\ref{chap1:thm:bisim} is a folk theorem in modal logic: for
a proof and discussion, see  \citet[Chapter 2.3]{hanbookmodal}.
Proposition~\ref{chap1:prop:smallmodel} is proved by
\cite{fagin:95a} as Theorem 3.2.2 (for the case $\mcx = \cal{K}$) and
Theorem 3.2.4 (for $\mcx = \mc{T}, \mc{S}4, \mc{KD}45$, and
$\mc{S}5$).  Proposition~\ref{chap1:prop:modelchecking} is
Proposition 3.2.1 by \cite{fagin:95a}. Theorem~\ref{chap1:thm:sat}
is proved by \cite{halpern:92a}.

Although the first proofs of completeness for multi-agent versions of
axiom systems of the form
$\axiom{X}_m$ and $\axiom{XC}_m$ are by \cite{halpern:92a}, the ideas go back much earlier.  In particular,
the basic canonical model construction goes back to 
\cite{Mak} (see \citet[Chapter 4]{blackburn:2001a} for a discussion),
while the idea for completeness of axiom
systems of the form $\axiom{XC}$ is already in the proof of \cite{KP} for proving completeness of dynamic logic.
Completeness for axiom systems of the form \axiom{XD} was proved by
\cite{fagin:92a} and by
\cite{vanderhoek:92a}.
A novel
proof is provided by \citet[Chapter 3]{wang:thesis}.
Theorem~\ref{thm:expressiveness} is part of logical folklore.
A proof of Theorem~\ref{thm:expsuccinct} was given by
\cite{Succ:AIJ}.

\paragraph{Acknowledgements} The authors are indebted to Cl\'{a}udia Nalon for a careful reading. 
Hans van Ditmarsch is also affiliated to IMSc, Chennai, as associated researcher, 
and he acknowledges support from European Research Council grant EPS 313360.
Joseph Y. Halpern was supported in part by NSF grants IIS-0911036 and CCF-1214844, 
by AFOSR grant FA9550-09-1-0266, by ARO grants W911NF-09-1-0281 and W911NF-14-1-0017, 
and by the Multidisciplinary University Research Initiative (MURI) program administered by the
AFOSR under grant FA9550-12-1-0040.


\refstart
\label{chap1:lastpage}









\end{document}

%% file: seesindex.tex
\index{modal logic|see{logic}}
\index{deontic logic|see{logic}}
\index{intensional logic|see{logic}}
\index{extensional logic|see{logic}}
\index{temporal logic|see{logic}}
\index{epistemic logic|see{logic}}
\index{linear time temporal logic|see{logic, \acro{LTL}}}
\index{computational tree logic|see{logic}}
\index{dynamic logic|see{logic}}
\index{intention logic|see{logic}}
\index{coalition logic|see{logic}}
\index{coalition logic!of propositional control|see{logic, \acro{CL-PC}}}
\index{ATL@\acro{ATL}|see{logic}}
\index{BDI logic @\acro{BDI} logic|see{logic}}
\index{CL-PC@\acro{CL-PC}|see{logic}}
\index{Alternating-time Temporal Logic|see{logic, \acro{ATL}}}
\index{modal operator|see{modality}}